\documentclass{article}

 \usepackage[preprint]{neurips_2026}

\usepackage[utf8]{inputenc} 
\usepackage[T1]{fontenc}    
\usepackage{hyperref}       
\usepackage{url}            
\usepackage{graphicx}       
\usepackage{booktabs}       
\usepackage{amsfonts}       
\usepackage{nicefrac}       
\usepackage{microtype}      
\usepackage{xcolor}         
\usepackage{amsmath}
\usepackage{amsthm}
{}
\newtheorem{proposition}{\bf Proposition}{}
{}
{}

\usepackage{algorithm}
\usepackage{algorithmic}

\title{DSSMs: State Space Models with Explicit Memory via Delay Differential Equations}

\author{%
  Yixiao Qian$^{1}$ \quad
  Song Chen$^{2,*}$ \quad
  Jiaxu Liu$^{3}$ \quad
  Shengze Cai$^{1}$ \quad
  Chao Xu$^{1,*}$\\
  $^{1}$College of Control Science and Engineering, Zhejiang University\\
  $^{2}$Department of Mathematics, National University of Singapore\\
  $^{3}$School of Mathematical Sciences, Zhejiang University\\
  \texttt{song.chen@nus.edu.sg, cxu@zju.edu.cn}\\
  $^{*}$Corresponding authors.
}

\begin{document}

\maketitle


\begin{abstract}
State Space Models (SSMs) have emerged as a powerful paradigm for efficient long-sequence modeling, offering parallel training and fast linear-time recurrent inference.
However, like other recurrent architectures, SSMs must compress an unbounded history into a fixed-size state, which limits context retention and makes precise retrieval over long-range context inherently difficult.
To overcome this limitation, we propose Delay State Space Models (DSSMs), a delay differential equation (DDE)-inspired extension of diagonal SSMs that augments discrete SSM recurrences with explicit delayed-state feedback.
Making explicit delayed feedback practical requires new stability parameterization, history management, and FFT-training tools.
We address these challenges with a practical discretization and parameterization grounded in a simple delay-independent stability condition.
To bypass direct time-domain kernel construction, we derive the DSSM transfer function and compute kernels in the frequency domain, using a kernel contour shift to suppress aliasing and recover accurate FFT training.
Empirically, DSSMs substantially improve targeted delayed-retrieval tasks while outperforming S4D on most standard sequence metrics and remaining close on the others.
\end{abstract}

\section{Introduction}

Sequence modeling is a core problem in modern deep learning.
Recently, State Space Models (SSMs) have drawn increasing attention as an efficient family of sequence models because they enable parallel training and fast linear-time inference while showing strong performance on long-sequence tasks~\citep{gu2021efficiently,gu2022parameterization,gu2024mamba}.
However, like other recurrent architectures, SSMs must compress past information into a finite recurrent state, where it is represented through exponentially decaying modes~\citep{wang2023expdecay,wang2023stablessm}, making exact copying and precise retrieval harder than in attention-based models~\citep{jelassi2024repeat}.
This limitation has motivated several complementary directions, including recall-oriented SSM layers such as H3~\citep{fu2022hungry}, more expressive state-space designs~\citep{lahoti2026mamba3}, and hybrid architectures that combine attention with SSMs, such as Jamba, Zamba, and Samba~\citep{lieber2024jamba,glorioso2024zamba,ren2024samba}.
This leaves open how to strengthen precise long-range memory while retaining efficient parallel training and fast inference: Transformers offer strong long-range memory and parallel training but incur expensive inference~\citep{shazeer2019fast,pope2023efficiently}, whereas standard SSMs achieve parallel training and efficient inference yet still struggle with precise long-range retrieval.

To overcome this bottleneck, we revisit the underlying dynamics and introduce time-delays into the state evolution.
Unlike standard ordinary differential equations (ODEs), delay differential equations (DDEs) explicitly incorporate historical states through delayed terms, thereby providing a natural mechanism for improving precise long-range memory.
Related work in scientific machine learning likewise suggests that introducing delay-aware continuous dynamics can improve expressive power and long-horizon prediction~\citep{zhu2021neural,gupta2021neural}.
Delayed dynamics, however, also create new challenges: stability analysis becomes more delicate, and delayed feedback breaks the simple time-domain kernel structure that underlies efficient FFT training in modern SSMs.
In this work, we show that both difficulties can be resolved within a practical neural sequence-modeling framework.

Our contributions can be summarized as follows:
\begin{itemize}
    \item \textbf{DDE-Inspired Delayed SSMs}: We propose Delay State Space Models (DSSMs), a DDE-inspired extension of diagonal SSMs that augments discrete SSM recurrences with explicit delayed-state feedback. We develop corresponding stability, discretization, and parameterization tools for sequence modeling.
    \item \textbf{Efficient FFT Training via Kernel Contour Shift}: We derive the exact transfer function of the delayed recurrence and develop a kernel contour shift technique that suppresses time-domain aliasing in frequency-domain training, enabling parallel training despite delayed feedback.
    \item \textbf{Empirical Validation}: We evaluate DSSMs on targeted delayed-memory tasks, efficiency benchmarks, and standard sequence modeling datasets, showing the delayed-memory mechanism and its transfer to broader sequence modeling settings.
\end{itemize}

\begin{figure*}[t]
    \centering
    \includegraphics[width=0.8\textwidth]{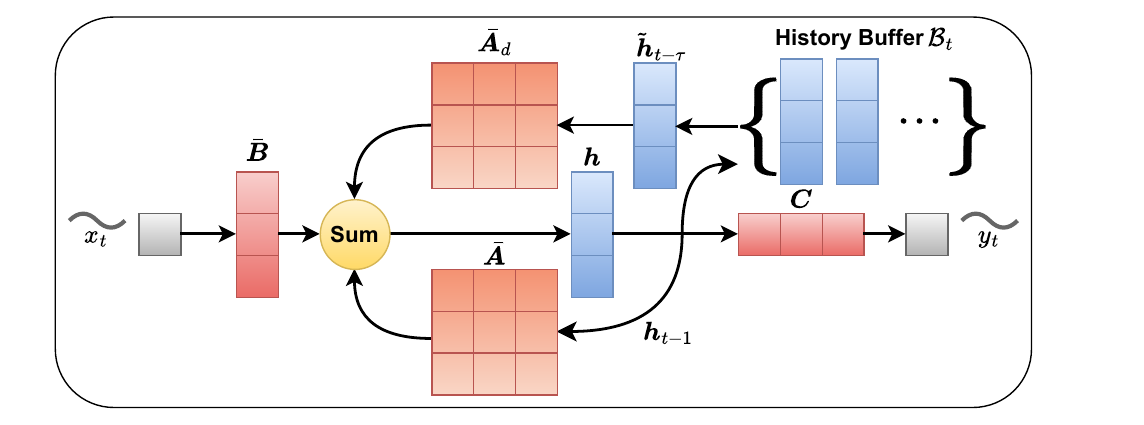}
    \caption{Overview of the discrete DSSM update. The state combines the input term, the main-branch term from $\boldsymbol{h}_{t-1}$, and the delayed-branch term from the buffered delayed state $\tilde{\boldsymbol{h}}_{t-\tau}$. 
    The formal discrete recurrence is developed in Sec.~\ref{sec:discretization}.}
    \label{fig:dssm-discrete-update}
\end{figure*}

\section{Background}
\label{sec:background}

\paragraph{Continuous State Space Models.}
SSMs describe the evolution of a continuous-time state $\boldsymbol{h}(t) \in \mathbb{C}^N$ driven by a 1-D input $x(t) \in \mathbb{R}$ and producing an output $y(t) \in \mathbb{R}$.
The system can be represented by either the linear ODE in Eq.~(\ref{eq:ssm_ode}) or by the convolution in Eq.~(\ref{eq:ssm_conv}):

\begin{minipage}{0.48\textwidth}
    \begin{equation}
        \label{eq:ssm_ode}
        \begin{aligned}
            \dot{\boldsymbol{h}}(t) &= \boldsymbol{A} \boldsymbol{h}(t) + \boldsymbol{B} x(t),\\
            y(t) &= \boldsymbol{C} \boldsymbol{h}(t).
        \end{aligned}
    \end{equation}
\end{minipage}
\hfill
\begin{minipage}{0.48\textwidth}
    \begin{equation}
        \label{eq:ssm_conv}
        \begin{aligned}
            K(t) &= \boldsymbol{C} e^{t \boldsymbol{A}} \boldsymbol{B},\\
            y(t) &= (K \ast x)(t).
        \end{aligned}
    \end{equation}
\end{minipage}

Here, $\boldsymbol{A} \in \mathbb{C}^{N \times N}$, $\boldsymbol{B} \in \mathbb{C}^{N \times 1}$, $\boldsymbol{C} \in \mathbb{C}^{1 \times N}$, $K(t) \in \mathbb{C}$, and the operations are defined as:
\begin{equation}
    e^{t\boldsymbol{A}} := \sum_{n=0}^\infty \frac{(t\boldsymbol{A})^n}{n!}, \quad (K \ast x)(t) := \int_0^t K(t - s) x(s) \mathrm{d}s.
\end{equation}

\paragraph{Discretization and Convolutional View.}
Applying the standard zero-order hold (ZOH) discretization~\citep{gu2021efficiently, gu2024mamba} to Eq.~(\ref{eq:ssm_ode}) yields the discrete recurrence:
\begin{equation}
    \label{eq:discrete_ssm_ode}
    \boldsymbol{h}_t = \bar{\boldsymbol{A}} \boldsymbol{h}_{t-1} + \bar{\boldsymbol{B}} x_t, \quad y_t = \boldsymbol{C} \boldsymbol{h}_t,
\end{equation}
where the discrete transition matrices are closed-form functions of the step size $\Delta t$:
\begin{equation}
    \label{eq:discrete_ssm_conv}
    \bar{\boldsymbol{A}} = \exp(\boldsymbol{A} \Delta t), \quad \bar{\boldsymbol{B}} = (\bar{\boldsymbol{A}} - \boldsymbol{I})\boldsymbol{A}^{-1} \boldsymbol{B}.
\end{equation}
While the recurrent form in Eq.~(\ref{eq:discrete_ssm_ode}) enables efficient $\mathcal{O}(1)$ step-by-step inference, its sequential dependence limits training parallelism. 
By the linear time-invariant (LTI) property, for an input sequence $\boldsymbol{x}$ and output sequence $\boldsymbol{y}$, the same system can be expressed as a global convolution:
\begin{equation}
    \boldsymbol{y} = \bar{\boldsymbol{K}} \ast \boldsymbol{x}, \quad \text{where } \bar{\boldsymbol{K}} = \left( \boldsymbol{C}\bar{\boldsymbol{B}}, \boldsymbol{C}\bar{\boldsymbol{A}}\bar{\boldsymbol{B}}, \boldsymbol{C}\bar{\boldsymbol{A}}^2\bar{\boldsymbol{B}}, \cdots, \boldsymbol{C}\bar{\boldsymbol{A}}^{L-1}\bar{\boldsymbol{B}} \right).
\end{equation}
The operation $\ast$ denotes the discrete convolution sum $y_t = \sum_{\ell=0}^t \bar{K}_\ell x_{t-\ell}$, where $\bar{K}_\ell:=\boldsymbol{C}\bar{\boldsymbol{A}}^{\ell}\bar{\boldsymbol{B}}$.

\paragraph{Structured SSMs.}
Applying SSMs to practical sequence modeling requires addressing three fundamental challenges:
(i) \textbf{Parameterization:} the system must be parameterized to ensure the ODE remains stable over long sequences;
(ii) \textbf{Initialization:} proper initialization of $\boldsymbol{A}$ is crucial for capturing long-term dependencies;
(iii) \textbf{Efficiency:} the representation must support parallelizable training to handle large-scale data.
S4~\citep{gu2021efficiently} established structured SSMs as a practical framework for long-range sequence modeling through a HiPPO-initialized~\citep{gu2020hippo} diagonal-plus-low-rank (DPLR) parameterization and an efficient Cauchy-kernel algorithm. 
Subsequent work further simplified and extended this framework in models such as DSS, S4D, S5, Liquid-S4, and Mamba~\citep{gupta2022diagonal,gu2022parameterization,smith2022simplified,hasani2022liquid,gu2024mamba}. 
Following the diagonal branch of this literature, we adopt diagonal transition matrices in DSSMs to preserve mode-wise decoupling and support efficient FFT-based training.

\paragraph{Memory Issues of SSMs.}
Recent empirical studies point to recall-related gaps between attention-based and state-space or recurrent models. H3 targets token recall and comparison in language modeling using shift and diagonal SSMs with multiplicative interactions~\citep{fu2022hungry}, while other work studies copying and position-structured retrieval benchmarks~\citep{arora2023zoology,jelassi2024repeat,ren2025exploring,waleffe2024empirical}. At a theoretical level, these gaps are consistent with the memory mechanism of ODE-based SSMs, which compress the entire past into a finite-dimensional recurrent state represented through exponentially decaying modes~\citep{wang2023expdecay,wang2023stablessm}. These observations motivate revisiting the underlying dynamical formulation itself.

\section{Delay State Space Models}
\label{sec:dde_ssm}

Motivated by these limitations, we revisit the underlying dynamics of SSMs. 
Standard SSMs are governed by ODEs, so past information can only be retained by being compressed into the current hidden state.
By contrast, DDEs incorporate delayed terms that provide explicit access to historical states. 
We therefore propose DSSMs as a DDE-inspired extension of diagonal SSMs that augments discrete SSM recurrences with explicit delayed-state feedback.

\subsection{Continuous Delay State Space Models}

A DSSM layer with hidden size $H$ consists of $H$ parallel channels, each with state size $N$.
We first present the single-channel formulation and return to the full layer in Sec.~\ref{sec:architecture}.

\paragraph{Continuous DSSMs.}
We first consider a single channel, modeled as the following single-input, single-output (SISO) continuous-time DDE:
\begin{equation}
    \label{eq:continuous_dde}
    \begin{aligned}
        \dot{\boldsymbol{h}}(t) = \boldsymbol{A} \boldsymbol{h}(t) + \boldsymbol{A}_d \boldsymbol{h}(t - \tau) + \boldsymbol{B} x(t), 
        \quad 
        y(t) = \boldsymbol{C} \boldsymbol{h}(t),
    \end{aligned}
\end{equation}
where $x(t) \in \mathbb{R}$, $y(t) \in \mathbb{R}$, and $\boldsymbol{h}(t) \in \mathbb{C}^N$ denote the input, output, and state, respectively. 
The learnable parameters are the state matrix $\boldsymbol{A} \in \mathbb{C}^{N \times N}$, delay matrix $\boldsymbol{A}_d \in \mathbb{C}^{N \times N}$, input projection $\boldsymbol{B} \in \mathbb{C}^{N \times 1}$, output projection $\boldsymbol{C} \in \mathbb{C}^{1 \times N}$, and time delay $\tau \in \mathbb{R}^+$. 
Following~\citep{gu2022parameterization}, we impose conjugate symmetry on the complex parameters to ensure real-valued outputs for real-valued inputs; details are provided in Appendix~\ref{sec:conjugate-symmetry}.

\paragraph{Diagonal Structure and Decoupling.} 
Following recent advances in structured SSMs, we restrict $\boldsymbol{A}$ and $\boldsymbol{A}_d$ to be diagonal. 
This decouples the $N$-dimensional system into $N$ independent scalar DDEs. 
Let $\boldsymbol{A} = \mathrm{diag}(\lambda_1, \cdots, \lambda_N)$ and $\boldsymbol{A}_d = \mathrm{diag}(\mu_1, \cdots, \mu_N)$. Then Eq.~\eqref{eq:continuous_dde} reduces to
\begin{equation}
    \dot{h}_n(t) = \lambda_n h_n(t) + \mu_n h_n(t - \tau) + b_n x(t),
    \quad n = 1,2,\cdots,N,
\end{equation}
where $h_n$ and $b_n$ denote the $n$-th components of $\boldsymbol{h}$ and $\boldsymbol{B}$, respectively.


The following sections instantiate a DDE-inspired discrete DSSM recurrence, together with its stable parameterization, history management, discretization, and initialization.

\subsection{History Management and Discretization}
\label{sec:discretization}

\paragraph{History Buffer.}
To implement delayed-state access in discrete time, we maintain a dense fixed-size buffer on the same grid as the recurrent update.
Before computing $\boldsymbol{h}_t$, the pre-update buffer stores the most recent available hidden states:
\begin{equation}
    \mathcal{B}_{t-1}
    =
    \left\{
        \boldsymbol{h}_{t-1-\ell}
        \mid
        0 \le \ell < \min(t, L_{\max})
    \right\}.
\end{equation}
After computing $\boldsymbol{h}_t$, it is written into the buffer to form $\mathcal{B}_t$.

\paragraph{Linear Interpolation.}
Since the learnable delay $\tau$ is continuous, the required delayed state generally does not lie exactly on the discrete grid.
We decompose the normalized delay as
\begin{equation}
    \frac{\tau}{\Delta t} := k + \alpha,
    \qquad
    k \in \mathbb{Z}^+, \ \alpha \in [0,1).
\end{equation}
Here, $k=\lfloor \tau/\Delta t \rfloor$ is treated as an integer index, while gradients flow through the fractional weight $\alpha$.
We define $\boldsymbol{h}_j=\boldsymbol{0}$ for all $j<0$, and use zero padding for positions outside the retained buffer:
\begin{equation}
    \hat{\boldsymbol{h}}_j
    =
    \begin{cases}
        \boldsymbol{h}_j, & \text{if } t-L_{\max} \le j < t, \\
        \boldsymbol{0}, & \text{otherwise},
    \end{cases}
\end{equation}
and approximate the delayed state by linear interpolation:
\begin{equation}
    \tilde{\boldsymbol{h}}_{t-\tau}
    =
    (1-\alpha)\hat{\boldsymbol{h}}_{t-k}
    +
    \alpha \hat{\boldsymbol{h}}_{t-k-1}.
\end{equation}

\begin{proposition}
\label{prop:discrete-delay-stability}
Consider the scalar discrete-time delayed system
\begin{equation}
    h_t = \bar{\lambda} h_{t-1} + \bar{\mu} \left[(1-\alpha) h_{t-k} + \alpha h_{t-k-1}\right],
    \qquad
    k \in \mathbb{Z}^+, \alpha \in [0,1).
\end{equation}
The zero-input system is asymptotically stable if
\begin{equation}
    |\bar{\lambda}| + |\bar{\mu}| < 1.
\end{equation}
\end{proposition}

\begin{proof}
A proof is given in Appendix~\ref{sec:proof-discrete-stability}.
\end{proof}

\paragraph{Discretization.}
To satisfy the discrete stability condition, we introduce a shared continuous base parameter $\zeta_n \in \mathbb{C}$ with $\mathrm{Re}(\zeta_n) < 0$. Its ZOH-style discretization defines the total discrete memory budget shared by the main and delayed branches:
\begin{equation}
    \bar{\zeta}_n
    =
    \exp(\zeta_n \Delta t),
    \qquad
    p_n = \mathrm{sigmoid}(r_n),
    \quad n = 1,2,\cdots,N.
\end{equation}
Its magnitude $\left|\bar{\zeta}_n\right| = \exp(\mathrm{Re}(\zeta_n)\Delta t) \in (0,1)$ gives the shared discrete budget, and $p_n$ allocates that budget to the delayed branch. We then parameterize the discrete coefficients as
\begin{equation}
    \bar{\lambda}_n
    =
    (1-p_n)\bar{\zeta}_n,
    \qquad
    \bar{\mu}_n
    =
    p_n \left|\bar{\zeta}_n\right| e^{i\theta_n}.
\end{equation}
By construction, $|\bar{\lambda}_n| = (1-p_n)|\bar{\zeta}_n|$ and $|\bar{\mu}_n| = p_n |\bar{\zeta}_n|$, hence
\begin{equation}
    |\bar{\lambda}_n| + |\bar{\mu}_n| = |\bar{\zeta}_n| < 1,
\end{equation}
which enforces the discrete stability condition mode-wise. Collecting these coefficients yields $\bar{\boldsymbol{A}} = \mathrm{diag}(\bar{\lambda}_1, \cdots, \bar{\lambda}_N)$ and $\bar{\boldsymbol{A}}_d = \mathrm{diag}(\bar{\mu}_1, \cdots, \bar{\mu}_N)$.
For the input projection, we use the standard diagonal SSM ZOH form for the non-delayed base dynamics,
\begin{equation}
    \bar{\boldsymbol{B}}
    =
    (\bar{\boldsymbol{A}} - \boldsymbol{I})\boldsymbol{A}^{-1} \boldsymbol{B},
\end{equation}
where $\boldsymbol{A} = \mathrm{diag}(\lambda_1, \cdots, \lambda_N)$ is the effective continuous-time matrix corresponding to $\bar{\boldsymbol{A}}$, with $\lambda_n := \zeta_n + \Delta t^{-1}\log(1-p_n)$. The delayed branch is parameterized directly in discrete time. The resulting vector recurrence is
\begin{equation}
    \label{eq:discrete_recurrence}
    \boldsymbol{h}_t
    =
    \bar{\boldsymbol{A}} \boldsymbol{h}_{t-1}
    +
    \bar{\boldsymbol{A}}_d \tilde{\boldsymbol{h}}_{t-\tau}
    +
    \bar{\boldsymbol{B}} x_t.
\end{equation}
Figure~\ref{fig:dssm-discrete-update} summarizes this discrete DSSM update.

\subsection{Parameterization}
\label{sec:parameterization}

\paragraph{DSSM Parameterization.}
We introduce unconstrained scalar parameters $\nu^{\Delta}, \nu^{\tau} \in \mathbb{R}$ and vector parameters $\boldsymbol{\nu}^{\text{re}}, \boldsymbol{\nu}^{\text{im}}, \boldsymbol{r}, \boldsymbol{\theta} \in \mathbb{R}^N$. Their non-bold counterparts with subscript $n$ denote the corresponding components. The model parameters are defined as follows:

\begin{itemize}
    \item \textbf{Step Size $\Delta t$}: We parameterize
    \begin{equation}
        \Delta t = \mathrm{softplus}(\nu^{\Delta}) > 0.
    \end{equation}

    \item \textbf{Shared Base Dynamics}: To ensure stability, we set
    \begin{equation}
        \zeta_n = -\exp(\nu^{\text{re}}_n) + i \nu^{\text{im}}_n,
    \end{equation}
    so that $\mathrm{Re}(\zeta_n) < 0$.

    \item \textbf{Branch Allocation and Delayed Feedback}: We define
    \begin{equation}
        \bar{\zeta}_n = \exp(\zeta_n \Delta t),
        \qquad
        p_n = \mathrm{sigmoid}(r_n),
        \quad n = 1,2,\cdots,N.
    \end{equation}
    We then split the shared discrete budget $\left|\bar{\zeta}_n\right|$ between the main and delayed branches via
    \begin{equation}
        \bar{\lambda}_n = (1-p_n)\bar{\zeta}_n,
        \qquad
        \bar{\mu}_n = p_n \left|\bar{\zeta}_n\right| e^{i\theta_n}.
    \end{equation}

    \item \textbf{Delay Location $\tau$}: We parameterize $\tau$ through a bounded ratio
    \begin{equation}
        q = \mathrm{sigmoid}(\nu^{\tau}), \qquad
        \tau = \Delta t + q \cdot (\tau_{\max} - \Delta t),
    \end{equation}
    where $\tau_{\max} = (L_{\max} - 1)\Delta t$.
\end{itemize}

\paragraph{Parameterization of $\boldsymbol{B}$.}
Following diagonal SSM practice~\citep{gu2022parameterization}, we use the explicit parameterization of $\boldsymbol{B}$; the absorbed form is mathematically equivalent, and details are provided in Appendix~\ref{sec:parameterization-equivalence}.

\subsection{Architecture and Initialization}
\label{sec:architecture}

\paragraph{Architecture.}
We instantiate the proposed DSSMs within the diagonal S4D backbone and denote the resulting model by DS4D. 
A DS4D layer with hidden size $H$ consists of $H$ parallel channels, each with $N$ diagonal state modes and one learnable delay shared across those modes.

\paragraph{Shared Base Dynamics Initialization.}
We follow the S4D initialization~\citep{gu2022parameterization} for the shared base dynamics, using either S4D-Inv or S4D-Lin:
\begin{equation}
    (\text{\textbf{S4D-Inv}}) \quad \zeta_n = -\frac{1}{2} + i \frac{N}{\pi} \left( \frac{N}{2n + 1} - 1\right),
    \qquad
    (\text{\textbf{S4D-Lin}}) \quad \zeta_n = - \frac{1}{2} + i \pi n.
\end{equation}

\paragraph{Delayed Branch Initialization.}
We initialize the delayed branch separately so that delayed feedback is present from the start of training:
\begin{equation}
    \theta_n = \omega + \epsilon_n,
\end{equation}
where $\epsilon_n$ is a small random perturbation, $\omega$ is chosen near $0$ or $\pi$, and $r_n$ is initialized so that $p_n$ starts from a small value. 
Initializing $\omega$ near $0$ biases the delayed branch toward reinforcing temporally aligned components, whereas initializing $\omega$ near $\pi$ biases it toward emphasizing mismatch or temporal change. 
Appendix~\ref{sec:phase-initialization-intuition} provides a qualitative intuition for this behavior.

\paragraph{Delay Location Initialization.}
For general sequence tasks, we initialize the normalized delay ratios on an evenly spaced grid over the admissible range:
\begin{equation}
    q_j = \epsilon + \frac{j - 1/2}{H}(1 - 2\epsilon),
    \qquad j=1,\ldots,H,
    \qquad
    \nu^{\tau}_{\pi(j)} = \log \left( \frac{q_j}{1 - q_j} \right),
\end{equation}
where $\pi$ is a random permutation of the channel indices.
When a task has a known characteristic delay scale, we instead use a task-specific initialization; for example, on image tasks, initializing delays around the image row width can serve as a useful heuristic for matching the delayed branch to the underlying spatial structure.

\section{Efficient FFT Training for Delay State Space Models}
\label{sec:fft_training}

Although Sec.~\ref{sec:dde_ssm} yields an efficient recurrent update for inference, sequential training remains prohibitively slow for long sequences.
Standard SSMs avoid this bottleneck through convolutional training, and in some cases through parallel scan.
In DSSMs, however, delayed feedback breaks the simple time-domain kernel formula available in ODE-based SSMs, while Appendix~\ref{sec:scan-based-parallelization} explains why parallel scan is also less natural for long-delay DSSMs.
We therefore derive a frequency-domain transfer function and combine it with a kernel contour shift to obtain an FFT training algorithm.

\paragraph{Why Standard FFT Convolution Fails for DSSMs.}
For standard SSMs, the LTI property yields a global discrete convolution, $\boldsymbol{y} = \bar{\boldsymbol{K}} \ast \boldsymbol{x}$, with a closed-form kernel $\bar{K}_t = \boldsymbol{C}\bar{\boldsymbol{A}}^t\bar{\boldsymbol{B}}$.
In DSSMs, delayed feedback introduces history accesses and recursive memory echoes. The kernel can still be generated recurrently, but this construction is sequential and less parallel-friendly, motivating the transfer-function route below.

\paragraph{Exact Z-Domain Transfer Function.}
We work directly in the Z-domain and derive the exact transfer function.

\begin{proposition}[Frequency-Domain Transfer Function]
\label{prop:transfer-function}
Consider the scalar delayed recurrence
\begin{equation}
    h_t = \bar{\lambda} h_{t-1} + \bar{\mu}\left[(1-\alpha)h_{t-k} + \alpha h_{t-k-1}\right] + \bar{b}x_t,
    \qquad
    y_t = c h_t,
\end{equation}
where $\bar{\lambda}$ and $\bar{\mu}$ are the discrete main-branch and delayed coefficients, $\bar{b}$ and $c$ are the scalar input and output coefficients, and $\tau = (k+\alpha)\Delta t$ with $k \in \mathbb{Z}^+$ and $\alpha \in [0,1)$.
Then the input-output transfer function is
\begin{equation}
    \label{eq:transfer-function}
    K(z)
    =
    c
    \left(
        1
        -
        z^{-1}\bar{\lambda}
        -
        z^{-k} D(z)\bar{\mu}
    \right)^{-1}
    \bar{b},
    \quad
    \text{where }
    D(z) = (1-\alpha) + \alpha z^{-1}.
\end{equation}
\end{proposition}

\begin{proof}
A proof is given in Appendix~\ref{sec:proof-transfer-function}.
\end{proof}

\paragraph{From Transfer Functions to FFT Training.}
A standard FFT-based route to recover a convolution kernel from a transfer function is to evaluate $K(z)$ on the length-$M$ FFT grid,
\begin{equation}
    \hat{\boldsymbol{K}} = [K(z_0), K(z_1), \cdots, K(z_{M-1})]^\top,
    \qquad
    z_m = \exp(i \frac{2\pi m}{M}),
\end{equation}
and apply IFFT, i.e., $\tilde{\boldsymbol{K}} = \mathrm{IFFT}(\hat{\boldsymbol{K}})$. 
Let $\bar{\boldsymbol{K}}$ denote the desired linear kernel associated with $K(z)$. The IFFT instead recovers the length-$M$ circular kernel, which is the time-domain alias of $\bar{\boldsymbol{K}}$:
\begin{equation}
    \tilde{K}_t = \sum_{r=0}^{\infty} \bar{K}_{t + rM}, \quad 0 \le t < M.
\end{equation}
Because delayed feedback creates long-lived memory echoes, the resulting time-domain aliasing $\sum_{r=1}^\infty \bar{K}_{t + rM}$ can be non-negligible. We therefore apply a kernel contour shift by damping the linear kernel as
\begin{equation}
    \bar{K}'_t = \rho^t \bar{K}_t,
    \qquad
    0 < \rho < 1,
\end{equation}
so that the time-domain aliasing $\sum_{r=1}^\infty \bar{K}'_{t + rM}$ becomes sufficiently small. Applying the same FFT recovery route to the shifted transfer function yields the damped circular kernel $\tilde{\boldsymbol{K}}'$, from which we recover the original linear kernel by reversing the damping, i.e., $\bar{K}_t \approx \rho^{-t}\tilde{K}'_t$. Smaller $\varepsilon_{\mathrm{shift}}$ reduces aliasing but increases restoration amplification through $\rho^{-t}$, so we choose it to balance aliasing suppression and floating-point error. Appendix~\ref{sec:kernel-contour-shift-details} provides the full derivation, including the aliasing relation, shifted coefficients, and restoration formula.

\paragraph{Algorithm Summary.}
Algorithm~\ref{alg:dssm_training} summarizes FFT training after discretization.
For a full DSSM layer, this procedure is applied independently to each channel.

\begin{algorithm}[htbp]
    \caption{Efficient FFT Training for DSSMs}
    \label{alg:dssm_training}
        \begin{algorithmic}[1]
        \REQUIRE 
            Input sequence $\boldsymbol{x} \in \mathbb{R}^{L}$;
            FFT length $M \ge 2L-1$;
            Discretized mode-wise coefficients $\bar{\boldsymbol{\lambda}}, \bar{\boldsymbol{\mu}}, \bar{\boldsymbol{b}}, \boldsymbol{c} \in \mathbb{C}^{N}$ for one channel;
            Delay decomposition $\tau/\Delta t = k + \alpha$;
            Shift tolerance $\varepsilon_{\mathrm{shift}}$.
        \ENSURE Output sequence $\boldsymbol{y} \in \mathbb{R}^{L}$.

        \STATE Set the damping radius and shifted dynamics:
        \begin{equation*}
            \rho \leftarrow \varepsilon_{\mathrm{shift}}^{1/M},
            \quad
            \bar{\boldsymbol{\lambda}}' \leftarrow \rho \bar{\boldsymbol{\lambda}},
            \quad
            \bar{\boldsymbol{\mu}}' \leftarrow \rho^{k}\bar{\boldsymbol{\mu}},
            \quad
            D'(z) \leftarrow (1-\alpha) + (\alpha \rho) z^{-1}.
        \end{equation*}
        \FOR{$m = 0$ \TO $M - 1$}
            \STATE Sample the shifted transfer function on the FFT grid $z_m = \exp(i \frac{2\pi m}{M})$:
            \begin{equation*}
                \boldsymbol{g}(z_m)
                =
                \left(
                    \boldsymbol{1}
                    -
                    z_m^{-1}\bar{\boldsymbol{\lambda}}'
                    -
                    z_m^{-k}D'(z_m)\bar{\boldsymbol{\mu}}'
                \right)^{-1},
                \quad
                K'(z_m)
                =
                \sum_{n=1}^{N} c_n g_n(z_m)\bar{b}_n.
            \end{equation*}
            Here, the inverse is taken elementwise across modes.
        \ENDFOR

        \STATE Apply IFFT to obtain the damped circular kernel $\tilde{\boldsymbol{K}}'$:
        \begin{equation*}
            \hat{\boldsymbol{K}}' \leftarrow [K'(z_0), K'(z_1), \cdots, K'(z_{M-1})]^\top,
            \quad
            \tilde{\boldsymbol{K}}' \leftarrow \mathrm{IFFT}(\hat{\boldsymbol{K}}').
        \end{equation*}
        \STATE Restore the length-$L$ linear kernel and zero-pad it to length $M$:
        \begin{equation*}
            \bar{K}_t \leftarrow \tilde{K}'_t\rho^{-t},
            \qquad
            t = 0,\cdots,L-1;
            \qquad
            \bar{K}_t \leftarrow 0,
            \quad
            t = L,\cdots,M-1.
        \end{equation*}

        \STATE Zero-pad the input to length $M$ and multiply in the frequency domain:
        \begin{equation*}
            \boldsymbol{x}^{\mathrm{pad}} \leftarrow [x_0,\cdots,x_{L-1},0,\cdots,0]^\top \in \mathbb{R}^{M},
            \quad
            \hat{\boldsymbol{Y}} = \mathrm{FFT}(\bar{\boldsymbol{K}}) \odot \mathrm{FFT}(\boldsymbol{x}^{\mathrm{pad}}).
        \end{equation*}
        \STATE Recover the length-$L$ output:
        \begin{equation*}
            y_t \leftarrow \left[\mathrm{IFFT}(\hat{\boldsymbol{Y}})\right]_t,
            \qquad
            t = 0,\cdots,L-1.
        \end{equation*}
    \end{algorithmic}
\end{algorithm}

\paragraph{Complexity Analysis.}
Per channel, the computation is dominated by element-wise evaluation of the diagonal transfer function across $N$ modes and $M$ frequencies, which costs $\mathcal{O}(MN)$, and by FFT/IFFT operations of length $M$, which cost $\mathcal{O}(M \log M)$. The per-channel complexity is therefore $\mathcal{O}(M(N + \log M))$, and a full layer with hidden size $H$ costs $\mathcal{O}(HM(N + \log M))$.
Implementation-level optimizations used in practice are described in Appendix~\ref{sec:implementation-details}.

\section{Experiments}
\label{sec:experiments}

We evaluate DS4D on targeted delayed-memory tasks, efficiency and memory cost, and standard benchmarks that test broader sequence-modeling performance. More details are provided in Appendix~\ref{sec:experimental-details}.

\subsection{Synthetic Tasks}
\label{sec:synthetic-tasks}

We use six synthetic benchmarks to test delayed retrieval and aggregation. The three target mappings---\emph{copy}, \emph{sum}, and \emph{weighted sum}---probe exact delayed recall, delayed aggregation, and content-sensitive delayed aggregation, each under both \emph{fixed-delay} and \emph{variable-delay} regimes. We compare DS4D against five representative baselines: LSTM, S4D, Mamba, a GPT-2 recipe, and a Llama recipe. 
Figure~\ref{fig:synthetic-learning-curves} shows the learning curves, while Appendix~\ref{sec:experimental-details} provides the full task definitions, data-generation details, and matched-parameter hyperparameters.

\begin{figure*}[htbp]
    \centering
    \includegraphics[width=\textwidth]{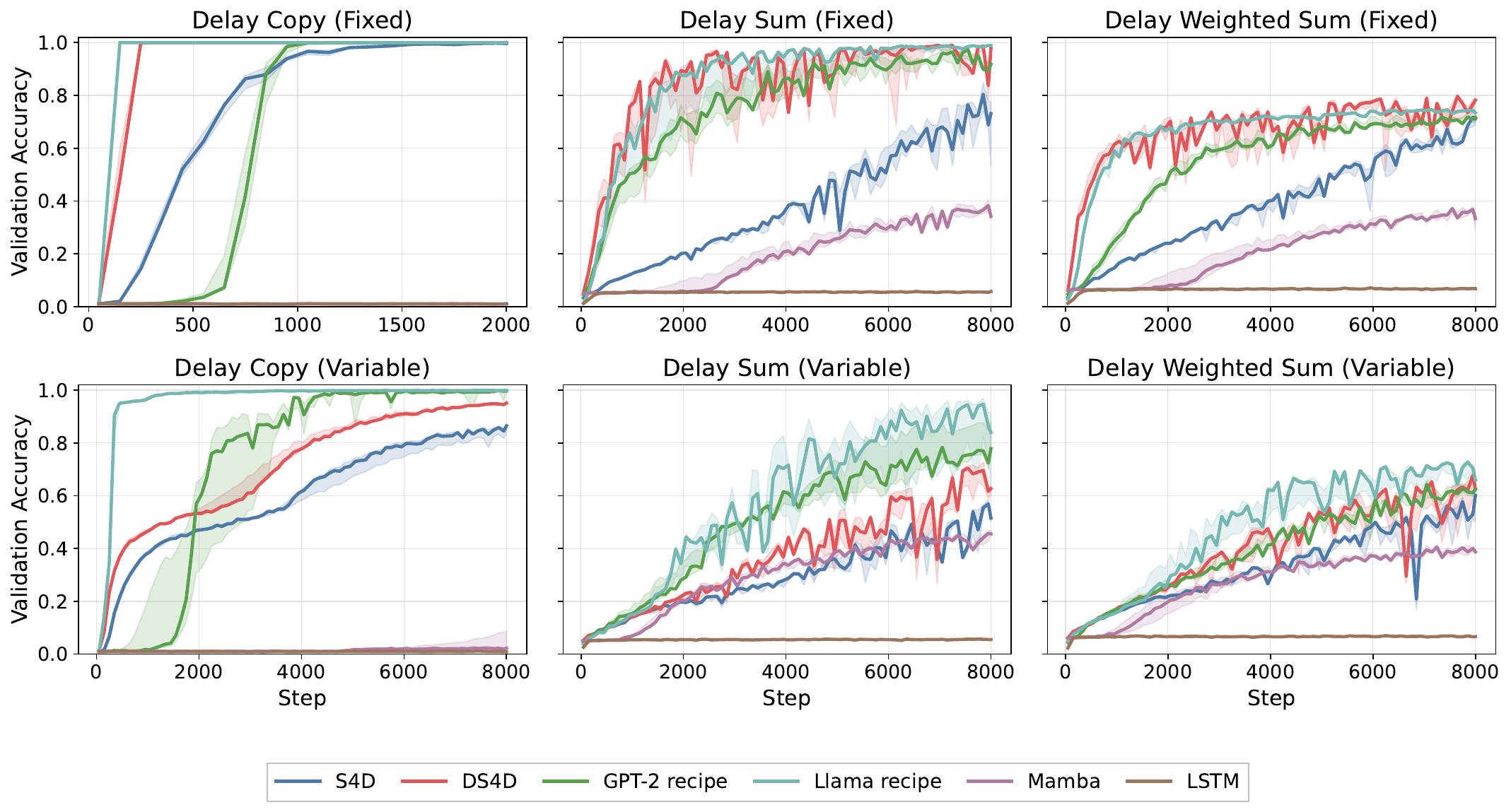}
    \caption{Validation trajectories on the six synthetic benchmarks using the best learning rate selected for each baseline. Curves show the median validation accuracy across three seeds, and shaded bands indicate interquartile ranges.}
    \label{fig:synthetic-learning-curves}
\end{figure*}

As shown in Figure~\ref{fig:synthetic-learning-curves}, DS4D converges rapidly on all three fixed-delay tasks and remains competitive with the strongest Transformer-based baselines, while under variable delays it is consistently more robust than S4D and Mamba. Table~\ref{tab:synthetic-mechanism} provides a more direct view of the learned mechanism: mean channel delays remain close to their initialization, suggesting that DS4D preserves a distributed bank of channel-wise delays rather than collapsing to a single task-specific lag, while nonzero changes in both $\tau$ and $\bar{\mu}$ indicate local retuning within that bank. The ablations support this interpretation, as collapsing the learned delay to its minimum one-step value or removing the delayed branch sharply degrades validation accuracy across all six tasks. 

\begin{table}[htbp]
\centering
\caption{Mechanistic evidence on synthetic DS4D checkpoints. 
(a) Mean channel delay, mean absolute delay shift $|\Delta \tau|$, and mean absolute delayed-branch change $|\Delta \bar{\mu}|$ from initialization to the best-validation checkpoint. 
(b) Validation accuracy under targeted ablations on all six synthetic tasks; $\tau_{\min}$ collapses the learned delay to one step, and $|\bar{\mu}|=0$ removes the delayed branch.}
\label{tab:synthetic-mechanism}
\footnotesize
\setlength{\tabcolsep}{2.8pt}
\begin{minipage}[t]{0.63\linewidth}
\centering
\textbf{(a) Parameter Learning Summary}
\vspace{2mm}
\begin{tabular}{lccc}
    \toprule
    Task & \shortstack{Mean $\tau$\\(init $\rightarrow$ best)} & \shortstack{Mean $|\Delta \tau|$} & \shortstack{Mean $|\Delta \bar{\mu}|$} \\
    \midrule
        Copy (F) & 544.0 $\rightarrow$ 543.9 & 2.6 & 0.004 \\
        Copy (V) & 544.0 $\rightarrow$ 545.7 & 5.0 & 0.018 \\
        Sum (F) & 544.0 $\rightarrow$ 543.9 & 4.4 & 0.005 \\
        Sum (V) & 288.0 $\rightarrow$ 288.1 & 2.3 & 0.008 \\
        W.Sum (F) & 544.0 $\rightarrow$ 543.8 & 3.8 & 0.005 \\
        W.Sum (V) & 288.0 $\rightarrow$ 288.0 & 2.4 & 0.008 \\
    \bottomrule
\end{tabular}
\end{minipage}
\hfill
\begin{minipage}[t]{0.33\linewidth}
\centering
\textbf{(b) Ablation Results}
\vspace{2mm}
\begin{tabular}{lccc}
    \toprule
    Task & Learned & $\tau_{\min}$ & $|\bar{\mu}|=0$ \\
    \midrule
        Copy (F) & 1.000 & 0.010 & 0.012 \\
        Copy (V) & 0.954 & 0.011 & 0.011 \\
        Sum (F) & 0.984 & 0.052 & 0.003 \\
        Sum (V) & 0.772 & 0.041 & 0.010 \\
        W.Sum (F) & 0.772 & 0.059 & 0.005 \\
        W.Sum (V) & 0.705 & 0.046 & 0.003 \\
    \bottomrule
\end{tabular}
\end{minipage}
\end{table}

\subsection{Efficiency and Accuracy}
\label{sec:efficiency-fft}

Figure~\ref{fig:efficiency} compares DS4D against S4D under matched-parameter and matched-hidden-size settings, reporting training throughput, decoding throughput, training memory, and decoding memory for sequence lengths from $1024$ to $8192$. 
FFT training remains close to S4D in both settings, consistent with the complexity derived in Sec.~\ref{sec:fft_training}.
During autoregressive inference, decoding throughput remains nearly flat as sequence length increases, matching the expected $\mathcal{O}(1)$ per-token compute. Training memory grows linearly with sequence length, while decoding memory scales as $\mathcal{O}(L)$ because DS4D retains the hidden-state history required for explicit delayed-state access. Detailed benchmark settings, including the official S4D CUDA backend, the recurrent S4D decode path, and our DS4D fast paths, are provided in Appendix~\ref{sec:efficiency-details}.

\begin{figure*}[htbp]
    \centering
    \includegraphics[width=\textwidth]{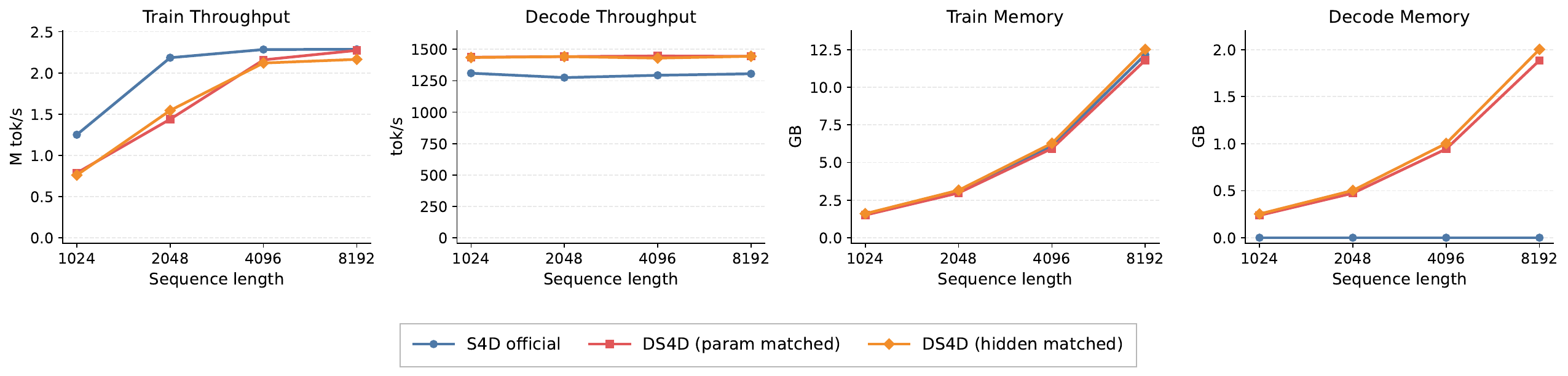}
    \caption{Efficiency of DS4D versus S4D under matched-parameter and matched-hidden-size settings. We report training throughput, decoding throughput, training memory, and decoding memory across sequence lengths. FFT training remains close to S4D, decoding throughput stays nearly length-independent, and decoding memory grows with the retained history horizon because DS4D caches past hidden states for explicit delayed-state access. S4D uses the complete official implementation, with the official CUDA Vandermonde backend for full-sequence training and the official recurrent step path for streaming decode.}
    \label{fig:efficiency}
\end{figure*}

To verify the FFT algorithm, we compare the FFT path against recurrent path on the same discretized DS4D model and inputs, and find that the numerical discrepancy is negligible. 
Detailed settings and per-length results are provided in Appendix~\ref{sec:efficiency-details}.

\subsection{General Sequence Modeling}
\label{sec:generalization}

Table~\ref{tab:general-seq} compares DS4D against S4D on standard sequence benchmarks spanning image, speech, medical time series, and long-range classification.
Although these benchmarks do not isolate delayed retrieval, they test whether explicit delayed memory preserves broad sequence-modeling ability.
DS4D outperforms S4D on most metrics and remains close on the others, with gains across multimodal and long-range tasks.
Full task details and hyperparameters are provided in Appendix~\ref{sec:experimental-details}.

\begin{table}[htbp]
    \centering
    \footnotesize
    \caption{General sequence modeling results for DS4D and S4D. Arrows indicate whether higher or lower is better; BIDMC reports RMSE. S4D results are literature-reported from \citet{gu2022parameterization} except sMNIST.}
    \label{tab:general-seq}
    \textbf{(a) Multimodal Tasks}\\[2mm]
    \resizebox{\textwidth}{!}{
    \begin{tabular}{lcccccc}
        \hline
        \textsc{Model} & \textsc{sMNIST} $\uparrow$ & \textsc{SC (AR)} $\uparrow$ & \textsc{SC (Bi.)} $\uparrow$ & \textsc{BIDMC-RR} $\downarrow$ & \textsc{BIDMC-HR} $\downarrow$ & \textsc{BIDMC-SpO2} $\downarrow$ \\
        \hline
        S4D & $99.47 \pm 0.001$ & $93.40 \pm 0.67$ & \textbf{\boldmath$96.18 \pm 0.27$} & \textbf{\boldmath$0.254 \pm 0.022$} & $0.373 \pm 0.024$ & $0.110 \pm 0.001$ \\
        DS4D & \textbf{\boldmath$99.52 \pm 0.000$} & \textbf{\boldmath$94.88 \pm 0.16$} & $96.04 \pm 0.33$ & $0.262 \pm 0.020$ & \textbf{\boldmath$0.365 \pm 0.014$} & \textbf{\boldmath$0.075 \pm 0.009$} \\
        \hline
    \end{tabular}
    }

    \vspace{2mm}
    \textbf{(b) Long Range Arena}\\[2mm]
    \resizebox{\textwidth}{!}{
    \begin{tabular}{lcccccc}
        \hline
        \textsc{Model} & \textsc{ListOps} $\uparrow$ & \textsc{Text} $\uparrow$ & \textsc{Retrieval} $\uparrow$ & \textsc{Image} $\uparrow$ & \textsc{Pathfinder} $\uparrow$ & \textsc{Path-X} $\uparrow$ \\
        \hline
        S4D & $60.18 \pm 0.35$ & $87.34 \pm 0.20$ & \textbf{\boldmath$91.09 \pm 0.01$} & $87.83 \pm 0.37$ & $93.78 \pm 0.25$ & $92.80$ \\
        DS4D & \textbf{\boldmath$61.35 \pm 0.56$} & \textbf{\boldmath$87.87 \pm 0.11$} & $91.03 \pm 0.11$ & \textbf{\boldmath$87.90 \pm 0.04$} & \textbf{\boldmath$94.05 \pm 0.36$} & \textbf{\boldmath$93.08$} \\
        \hline
    \end{tabular}
    }
\end{table}

\section{Limitations}
\label{sec:limitations}

When delayed structure is not beneficial for a task, the gains of DS4D over S4D may be limited.
Decoding memory still scales with the bounded history horizon $L_{\max}$ rather than remaining strictly $\mathcal{O}(1)$.
For fixed $L_{\max}$, however, this cost does not grow with the total sequence length, unlike Transformer KV caches.
FFT training relies on a kernel contour shift that balances residual time-domain aliasing against floating-point restoration error.
The delayed dynamics are not yet fully understood; in particular, delayed-branch initialization and delay-location design need more systematic study.
Finally, we do not yet evaluate DS4D on large-scale autoregressive language modeling, and extending DSSMs with selective or input-dependent mechanisms remains future work.

\section{Conclusion}

We introduced Delay State Space Models (DSSMs), a DDE-inspired extension of diagonal SSMs that augments discrete SSM recurrences with explicit delayed-state feedback while retaining efficient FFT training and fast linear-time inference. 
We developed stable parameterizations, derived an exact frequency-domain transfer function, and proposed a practical FFT training algorithm based on kernel contour shift. 
Empirically, DS4D substantially improves targeted delayed-memory tasks and outperforms S4D on most standard sequence metrics while remaining close on the others. 


\bibliography{references}


\appendix

\section{Method Details}

\subsection{Conjugate Symmetry and Real-Valued Outputs}
\label{sec:conjugate-symmetry}

\begin{proposition}[Conjugate Symmetry]
A sufficient and standard way to guarantee real-valued outputs for real-valued inputs is to impose conjugate-pair symmetry on the complex parameters. Specifically, under the explicit input parameterization, for any mode $n$ with complex parameters, we include a corresponding mode $m$ such that
\begin{equation}
    \lambda_m = \lambda_n^*, \quad \mu_m = \mu_n^*, \quad b_m = b_n^*, \quad \text{and} \quad c_m = c_n^*.
\end{equation}
Here $(\cdot)^*$ denotes complex conjugation.
\end{proposition}

\begin{proof}
The total system output is the sum of outputs from independent scalar modes: $y(t) = \sum_{n=1}^N y_n(t)$, where $y_n(t) = c_n h_n(t)$.
Let $H_n(s)$ and $X(s)$ denote the Laplace transforms of the state $h_n(t)$ and the input $x(t)$, respectively. 
Applying the Laplace transform to the scalar DDE (assuming zero initial conditions) yields:
\begin{equation}
    s H_n(s) = \lambda_n H_n(s) + \mu_n e^{-\tau s} H_n(s) + b_n X(s).
\end{equation}
Rearranging terms, the transfer function $T_n(s) = H_n(s)/X(s)$ for the $n$-th mode is:
\begin{equation}
    T_n(s) = \frac{b_n}{s - \lambda_n - \mu_n e^{-\tau s}}.
\end{equation}
The global transfer function of the system, mapping input $X(s)$ to output $Y(s)$, is:
\begin{equation}
    G(s) = \sum_{n=1}^N c_n T_n(s) = \sum_{n=1}^N \frac{c_n b_n}{s - \lambda_n - \mu_n e^{-\tau s}}.
\end{equation}
Consider the term associated with index $n$, denoted $G_n(s) = c_n T_n(s)$. Its conjugate evaluated at $s^*$ is:
\begin{equation}
    G_n(s^*)^* = \left( \frac{c_n b_n}{s^* - \lambda_n - \mu_n e^{-\tau s^*}} \right)^* = \frac{c_n^*\,b_n^*}{s - \lambda_n^* - \mu_n^* e^{-\tau s}}.
\end{equation}
This expression represents a dynamical system with parameters $\lambda_n^*, \mu_n^*, b_n^*, c_n^*$. 
Thus, if every complex mode is paired with its conjugate, the corresponding transfer-function terms satisfy $G_m(s)=G_n(s^*)^*$.
The summed transfer function is therefore closed under conjugation and satisfies $G(s^*) = G(s)^*$, which corresponds to a real-valued impulse response and preserves real-valued outputs for real-valued inputs.
\end{proof}

\subsection{Proof of the Discrete Stability Criterion}
\label{sec:proof-discrete-stability}

\begin{proof}[Proof of Proposition~\ref{prop:discrete-delay-stability}]
For fixed $k \in \mathbb{Z}^+$ and $\alpha \in [0,1)$, the zero-input scalar recurrence is
\begin{equation}
    h_t = \bar{\lambda} h_{t-1} + \bar{\mu}\left[(1-\alpha) h_{t-k} + \alpha h_{t-k-1}\right].
\end{equation}
Asymptotic stability is equivalent to requiring that all roots of the corresponding characteristic equation lie strictly inside the unit disk.
We seek modal solutions of the form $h_t = z^t$ with $z \in \mathbb{C}$ and $z \neq 0$. 
Substituting $h_t = z^t$ into the recurrence gives
\begin{equation}
    z^t = \bar{\lambda} z^{t-1} + \bar{\mu}\left[(1-\alpha) z^{t-k} + \alpha z^{t-k-1}\right].
\end{equation}
Dividing both sides by $z^{t-k-1}$ yields the characteristic equation
\begin{equation}
    z^{k+1} = \bar{\lambda} z^k + \bar{\mu}\left[(1-\alpha) z + \alpha\right].
\end{equation}

We now prove that under the condition $|\bar{\lambda}| + |\bar{\mu}| < 1$, this equation cannot admit a root with $|z| \ge 1$. 
Suppose, for contradiction, that there exists such a root. 
Dividing the characteristic equation by $z^{k+1}$ gives
\begin{equation}
    1 = \bar{\lambda} z^{-1} + \bar{\mu}(1-\alpha) z^{-k} + \bar{\mu}\alpha z^{-k-1}.
\end{equation}
Taking absolute values and applying the triangle inequality, we obtain
\begin{equation}
    1
    \le
    |\bar{\lambda}| |z|^{-1}
    +
    |\bar{\mu}|(1-\alpha)|z|^{-k}
    +
    |\bar{\mu}|\alpha |z|^{-k-1}.
\end{equation}
Since $|z| \ge 1$, we have $|z|^{-1} \le 1$, $|z|^{-k} \le 1$, and $|z|^{-k-1} \le 1$. Using $(1-\alpha) + \alpha = 1$, this implies
\begin{equation}
    1 \le |\bar{\lambda}| + |\bar{\mu}|(1-\alpha) + |\bar{\mu}|\alpha = |\bar{\lambda}| + |\bar{\mu}|,
\end{equation}
which contradicts the assumption $|\bar{\lambda}| + |\bar{\mu}| < 1$.
Hence, every characteristic root must satisfy $|z| < 1$. 
Therefore, the zero-input recurrence is asymptotically stable.
\end{proof}

\subsection{Input Parameterization Equivalence}
\label{sec:parameterization-equivalence}

\begin{proposition}[Input Parameterization Equivalence]
For any diagonal DSSM system $\mathcal{S} = (\boldsymbol{A}, \boldsymbol{A}_d, \boldsymbol{B}, \boldsymbol{C}, \tau)$ where the input projection elements $b_n \neq 0$ for all $n$, there exists an equivalent absorbed-input system $\mathcal{S}' = (\boldsymbol{A}, \boldsymbol{A}_d, \boldsymbol{1}, \boldsymbol{C}', \tau)$ such that their outputs are identical for any input $x(t)$, with $c'_n = c_n b_n$.
\end{proposition}

\begin{proof}
Consider the scalar DDE governing the $n$-th state mode in system $\mathcal{S}$:
\begin{equation}
    \dot{h}_n(t) = \lambda_n h_n(t) + \mu_n h_n(t-\tau) + b_n x(t).
\end{equation}
Since the system is linear, we can define a scaled state variable $\tilde{h}_n(t) = h_n(t) / b_n$. Substituting $h_n(t) = b_n \tilde{h}_n(t)$ into the dynamic equation and dividing both sides by the scalar $b_n$, we obtain:
\begin{equation}
    \dot{\tilde{h}}_n(t) = \lambda_n \tilde{h}_n(t) + \mu_n \tilde{h}_n(t-\tau) + 1 \cdot x(t).
\end{equation}
The output of the original system is given by $y(t) = \sum_{n=1}^N c_n h_n(t)$. Substituting the scaled state yields:
\begin{equation}
    y(t) = \sum_{n=1}^N (c_n b_n) \tilde{h}_n(t).
\end{equation}
By constructing a new output weight vector $\boldsymbol{C}'$ such that $c'_n = c_n b_n$, and setting the new input weight to unity ($\boldsymbol{B}' = \boldsymbol{1}$), we derive the system $\mathcal{S}'$. This new system generates the exact same input-output mapping as $\mathcal{S}$. 
Therefore, the explicit and absorbed input parameterizations are equivalent up to a mode-wise rescaling.
\end{proof}

\subsection{Phase Initialization Intuition}
\label{sec:phase-initialization-intuition}

The intuition behind delayed-phase initialization is easiest to state under a locally smooth hidden-trajectory assumption: over short windows, background or context features vary more slowly than isolated task-relevant events. 
In this regime, $\theta \approx 0$ biases the delayed branch toward reinforcing compatible slowly varying components, whereas $\theta \approx \pi$ biases it toward sign-inverted cancellation and therefore greater sensitivity to mismatch or temporal change. 

To visualize these two biases, Fig.~\ref{fig:phase-init-intuition} compares simple delayed recurrences under the same stable main-branch dynamics but different delayed phases, $\theta \approx 0$ and $\theta \approx \pi$. 
Here we use a real-valued main-branch coefficient to isolate the effect of the delayed phase itself.
Both panels use the scalar discrete system $h_t = \bar{\lambda} h_{t-1} + \bar{\mu} h_{t-k} + x_t$ with real $\bar{\lambda}=0.75$, integer delay $k=16$, and identity input/output projections. 
Panel (a) uses $\bar{\mu}=\pm 0.10$ under a constant background with a sustained pulse event, while panel (b) uses $\bar{\mu}=\pm 0.20$ under the same background with a narrow Gaussian spike. 
In both panels, dashed curves denote the response of the same system to the background-only input, while solid curves denote the response to the background-plus-event input, and the marked positions on the horizontal axis indicate the event onset time $t_0$ and its delayed reappearance at $t_0 + \tau$.

\begin{figure*}[htbp]
    \centering
    \includegraphics[width=0.98\textwidth]{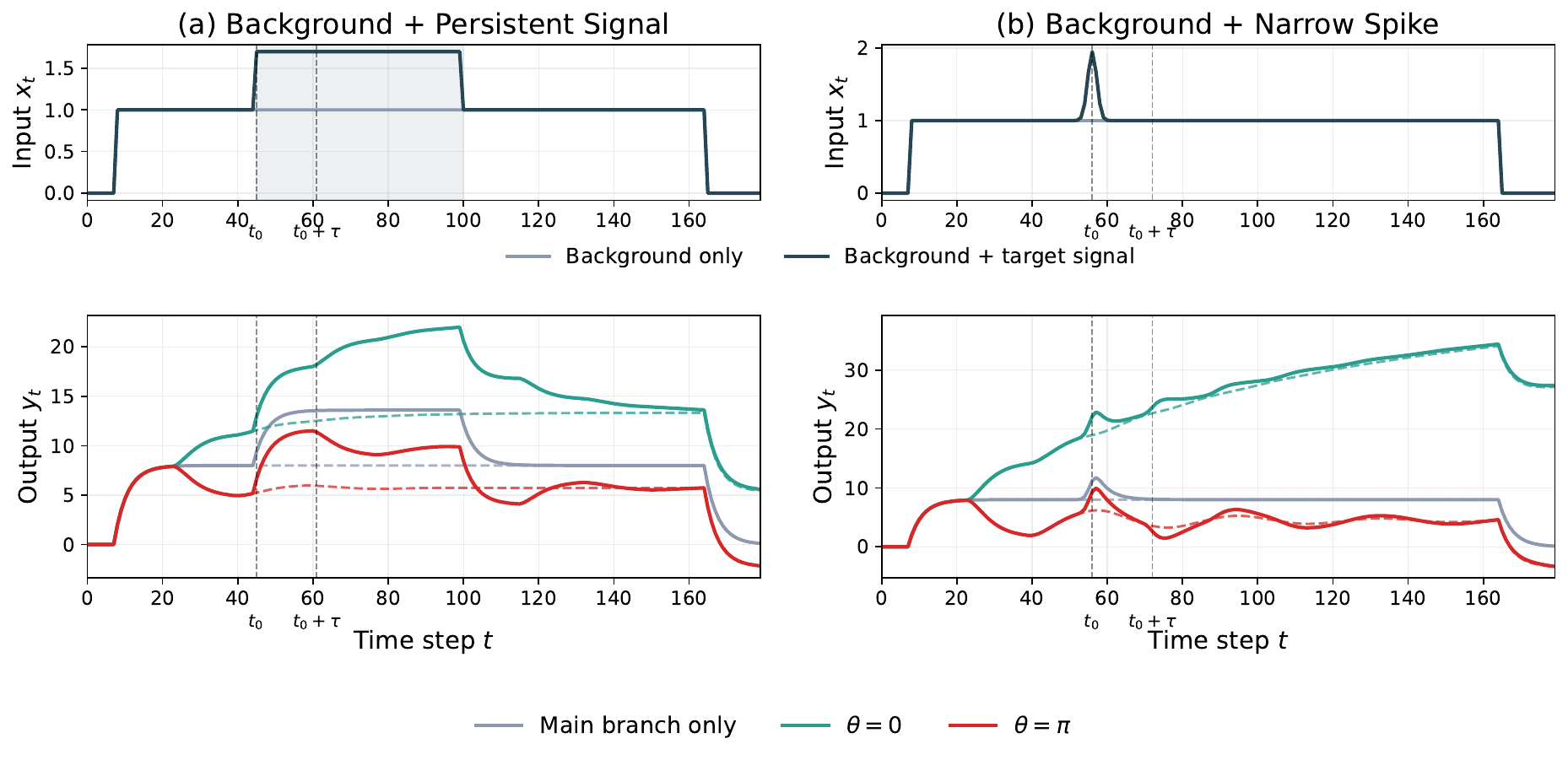}
    \caption{Qualitative effect of delayed-phase initialization under two simple backgrounds. (a) Background + persistent signal: initialization near $\theta \approx 0$ causes the delayed branch to reinforce temporally aligned signal components, so the persistent event is progressively amplified after $t_0+\tau$, whereas initialization near $\theta \approx \pi$ suppresses signal components that remain similar across the delay horizon and therefore emphasizes temporal change. (b) Background + narrow spike: at $t_0+\tau$, the delayed branch re-injects the localized event with a phase determined by $\theta$, producing a constructive delayed echo near $\theta \approx 0$ and a sign-inverted delayed echo near $\theta \approx \pi$. Together, the two panels illustrate the broader bias that $\theta \approx 0$ favors delayed aggregation of compatible signal components, while $\theta \approx \pi$ favors sensitivity to mismatch and temporal change.}
    \label{fig:phase-init-intuition}
\end{figure*}

\subsection{Relation to Scan-Based Parallelization}
\label{sec:scan-based-parallelization}

Parallel scan is most natural for first-order recurrences of the form
\begin{equation}
    h_t = \bar{\lambda} h_{t-1} + \bar{b} x_t,
\end{equation}
where the update acts on a constant-size state and the associated combine operator remains associative. This is the setting targeted by scan-based SSM implementations such as S5-style discretizations.

For simplicity, we suppress the interpolation term and consider the integer-delay recurrence
\begin{equation}
    h_t = \bar{\lambda} h_{t-1} + \bar{\mu} h_{t-k} + \bar{b} x_t
\end{equation}
which is not first-order Markovian in the original state. Recasting it into a scan-friendly first-order form requires augmenting the state with the past $k$ delay slots, e.g.,
\begin{equation}
    \boldsymbol{H}_t = [h_t, h_{t-1}, \cdots, h_{t-k+1}]^\top.
\end{equation}
The resulting update again has first-order form, but it now acts on a state whose dimension scales with the delay horizon rather than remaining constant.

For small fixed delays this reformulation is still possible, but for long delays it removes the lightweight constant-state advantage that makes scan-based SSMs attractive. 
In particular, the scan combine step no longer acts on scalar or constant-size diagonal updates, but on $k$-dimensional transition objects associated with the augmented state. 
A naive implementation would therefore compose $k \times k$ transition matrices, while even structure-exploiting implementations must still carry state and intermediate operators whose size grows with $k$, leading to substantially higher compute and memory cost than in standard first-order SSMs.
In contrast, discretized DSSMs remain linear time-invariant, so the exact transfer-function route of Sec.~\ref{sec:fft_training} continues to exist and provides a more natural exact path to parallel training.

\subsection{Proof of Transfer Function}
\label{sec:proof-transfer-function}

\begin{proof}[Proof of Proposition \ref{prop:transfer-function}]
Let $h_t \in \mathbb{C}$ and $x_t \in \mathbb{R}$ denote the scalar state and input at discrete step $t$.
With $\tau = (k+\alpha)\Delta t$, the scalar recurrence is
\begin{equation}
    h_t = \bar{\lambda} h_{t-1} + \bar{\mu}\left[(1-\alpha)h_{t-k} + \alpha h_{t-k-1}\right] + \bar{b}x_t.
\end{equation}
Applying the Z-transform and using the standard integer-shift rules
\[
\mathcal{Z}\{h_{t-1}\}=z^{-1}H(z),\qquad
\mathcal{Z}\{h_{t-k}\}=z^{-k}H(z),\qquad
\mathcal{Z}\{h_{t-k-1}\}=z^{-k-1}H(z),
\]
gives
\begin{equation}
    H(z)
    =
    \bar{\lambda}z^{-1}H(z)
    +
    \bar{\mu}z^{-k}\left((1-\alpha)+\alpha z^{-1}\right)H(z)
    +
    \bar{b}X(z).
\end{equation}
Solving for $H(z)$ yields
\begin{equation}
    H(z)
    =
    \left(
        1
        -
        \bar{\lambda}z^{-1}
        -
        \bar{\mu}z^{-k}\left((1-\alpha)+\alpha z^{-1}\right)
    \right)^{-1}
    \bar{b}X(z).
\end{equation}
Since $y_t = c h_t$, the transfer function for this scalar channel is
\begin{equation}
    K(z)
    =
    c
    \left(
        1
        -
        \bar{\lambda}z^{-1}
        -
        \bar{\mu}z^{-k}\left((1-\alpha)+\alpha z^{-1}\right)
    \right)^{-1}
        \bar{b}.
\end{equation}
\end{proof}

\subsection{From Transfer Functions to FFT Training}
\label{sec:kernel-contour-shift-details}

\paragraph{From Transfer Functions to Convolution Kernels.}
A standard FFT-based route from a transfer function to a convolution kernel is to evaluate $K(z)$ on the FFT grid $z_m = \exp(i \frac{2\pi m}{M})$ for $m=0,\cdots,M-1$, forming the sampled frequency response
\begin{equation}
  \hat{\boldsymbol{K}}
  =
  \left[
      K(z_0), K(z_1), \cdots, K(z_{M-1})
  \right]^\top.
\end{equation}
Applying IFFT then yields the corresponding length-$M$ \emph{circular convolution kernel}
\begin{equation}
  \tilde{\boldsymbol{K}}
  =
  \left[
      \tilde{K}_0, \tilde{K}_1, \cdots, \tilde{K}_{M-1}
  \right]^\top
  =
  \mathrm{IFFT}(\hat{\boldsymbol{K}}).
\end{equation}
However, this standard procedure does not directly recover the desired linear kernel $\bar{\boldsymbol{K}}$. Instead, the circular kernel is a time-domain alias of the linear kernel:
\begin{equation}
    \tilde{K}_t = \sum_{r=0}^{\infty} \bar{K}_{t + rM}, \quad 0 \le t < M.
\end{equation}
Due to the memory echoes characteristic of delay systems, the time-domain aliasing tail $\sum_{r \ge 1} \bar{K}_{t+rM}$ may remain non-negligible.

\paragraph{Kernel Contour Shift.}
To suppress time-domain aliasing, we damp the linear kernel as $\bar{K}'_t = \rho^t \bar{K}_t$ with $0 < \rho < 1$.
Let $\rho = \exp(-\sigma \Delta t)$ be the damping factor. For an FFT window of length $M$, we choose $\sigma$ so that
\begin{equation}
    \rho^M \approx \varepsilon_{\mathrm{shift}},
\end{equation}
where $\varepsilon_{\mathrm{shift}} \ll 1$ is a small decay target.
The corresponding shifted discrete dynamics are
\begin{equation}
    \bar{\lambda}' = \rho \bar{\lambda},
    \qquad
    \bar{\mu}' = \rho^{k}\bar{\mu},
    \qquad
    D'(z) = (1-\alpha) + (\alpha \rho) z^{-1}.
\end{equation}
We then evaluate the shifted transfer function
\begin{equation}
    K'(z)
    =
    c
    \left(
        1
        -
        z^{-1}\bar{\lambda}'
        -
        z^{-k}D'(z)\bar{\mu}'
    \right)^{-1}
    \bar{b}.
\end{equation}
Applying IFFT yields the corresponding damped circular kernel $\tilde{\boldsymbol{K}}'$:
\begin{equation}
    \tilde{K}'_t 
    = \sum_{r=0}^{\infty} \bar{K}'_{t+rM}
    = \sum_{r=0}^{\infty} \rho^{t+rM} \bar{K}_{t+rM},
    \qquad
    0 \le t < M.
\end{equation}
Here $\bar{K}'_t$ denotes the damped linear time-domain kernel associated with $K'(z)$, while $\bar{K}_t$ denotes the original undamped linear kernel associated with $K(z)$.
After the IFFT, we recover the linear kernel by reversing the damping in time:
\begin{equation}
    \rho^{-t}\tilde{K}'_t
    =
    \bar{K}_t + \sum_{r=1}^{\infty} \rho^{rM} \bar{K}_{t+rM}
    \quad
    \Rightarrow
    \quad
    \bar{K}_t \approx \rho^{-t}\tilde{K}'_t.
\end{equation}
Since $\rho^M \approx \varepsilon_{\mathrm{shift}}$, each aliased tail term is suppressed by an additional factor of approximately $\varepsilon_{\mathrm{shift}}$. However, reversing the damping multiplies numerical errors by $\rho^{-t}$, so very small $\varepsilon_{\mathrm{shift}}$ can amplify floating-point error. In practice, $\varepsilon_{\mathrm{shift}}$ is chosen to balance aliasing suppression and restoration error.

\section{Implementation Details}
\label{sec:implementation-details}

We use two CUDA fast paths in our implementation, one for FFT training and one for autoregressive decoding.

\paragraph{FFT Transfer Construction.}
For FFT training, we use a custom CUDA kernel to build the contour-shifted transfer function on the frequency grid from the shifted form of Eq.~\eqref{eq:transfer-function}.
Instead of materializing the shifted coefficients and delayed terms as separate intermediate tensors in PyTorch, the CUDA kernel directly fuses the following steps on the FFT grid:
\begin{itemize}
    \item construct the contour-shifted main-branch and delayed coefficients;
    \item combine the delayed interpolation factors with the frequency terms $z^{-k}$ and $z^{-k-1}$;
    \item form the denominator of the rational transfer function for each mode and frequency;
    \item accumulate the mode-wise contributions into the final sampled transfer function.
\end{itemize}
This fused construction is the main implementation optimization behind the FFT-training results in Sec.~\ref{sec:efficiency-fft}.

\paragraph{Fused Decode Kernel.}
Given the previous state and the history buffer, a single kernel launch performs the following operations:
\begin{itemize}
    \item locate the two delayed anchors corresponding to $t-k$ and $t-k-1$ in the history buffer;
    \item gather the two historical states and apply the validity mask induced by the current cache position;
    \item linearly interpolate them to form the delayed state used by the recurrence;
    \item update the next state using the main branch, delayed branch, and input term;
    \item write the updated state back into the current history-buffer slot;
    \item compute the final readout.
\end{itemize}
This removes a long sequence of small gather, masking, interpolation, recurrence, and writeback kernels from the decode hot path, which is the main reason the optimized DS4D decode remains competitive with S4D.

\paragraph{Caching and Precomputation.}
In addition to the CUDA kernels, we cache recurrent constants and frequency-domain geometry whenever possible. 
Specifically, inference reuses discretized coefficients and interpolation weights, while FFT training reuses frequency-grid quantities associated with the current sequence length, dtype, and device. These caches further reduce setup overhead without changing the underlying model or algorithm.

\section{Experiment Configurations}
\label{sec:experimental-details}

\subsection{Baselines}

We compare DS4D against a comprehensive suite of baselines ranging from classic recurrent networks to state-of-the-art Transformers and recent SSM variants. 
All models are either implemented in PyTorch or imported from the transformers library. 
Unless otherwise stated, we control the model size by fixing the number of layers and changing the hidden dimension $H$ to ensure a fair comparison across different architectures.

\begin{itemize}
    \item \textbf{LSTM}~\citep{hochreiter1997long}: The classic long short-term memory network. It utilizes gating mechanisms to mitigate vanishing gradients but relies on sequential processing, limiting training parallelization.
    \item \textbf{S4D}~\citep{gu2022parameterization}: The diagonal structured SSM. It relies on linear ODEs ($\dot{h} = Ah + Bx$) and HiPPO initialization to model long-range dependencies.
    \item \textbf{Mamba}~\citep{gu2024mamba}: A selective state space model where system parameters are input-dependent. Mamba represents the current state-of-the-art in linear-time sequence modeling.
    \item \textbf{GPT-2 recipe}~\citep{radford2019language}: A standard causal Transformer recipe with absolute positional embeddings and pre-layernorm.
    \item \textbf{Llama recipe}~\citep{touvron2023llama}: A modern Transformer recipe incorporating rotary positional embeddings (RoPE)~\citep{su2024roformer}, RMSNorm, and SiLU activation.
\end{itemize}

\subsection{Synthetic Tasks}

All synthetic tasks use the same prompt-completion format
\begin{equation}
    \boldsymbol{x} = [\boldsymbol{x}_{\text{data}}, \boldsymbol{x}_{\text{lag}}, x_{\text{go}}, \boldsymbol{x}_{\text{target}}],
\end{equation}
where $\boldsymbol{x}_{\text{data}} \in \mathbb{Z}^{L_{\text{data}}}$ is a random source sequence, $\boldsymbol{x}_{\text{lag}}$ is a zero-padding block of length $L_{\text{lag}}$, $x_{\text{go}}$ is a trigger token, and $\boldsymbol{x}_{\text{target}}$ is the target sequence. We refer to $L_{\text{lag}}$ as the padding lag; the absolute source-target offset also depends on the source position and causal next-token alignment. Training uses a standard causal next-token objective, but the loss is computed only on the $\boldsymbol{x}_{\text{target}}$ segment so that the model is evaluated purely on delayed retrieval or aggregation.

The three target mappings differ only in how $\boldsymbol{x}_{\text{target}}$ is constructed from $\boldsymbol{x}_{\text{data}}$:
\begin{itemize}
    \item \textbf{Delay Copy}: $\boldsymbol{x}_{\text{target}} = \boldsymbol{x}_{\text{data}}$.
    \item \textbf{Delay Sum}: each target token is a sliding-window sum,
    \begin{equation}
        \boldsymbol{x}_{\text{target}}[t] = \sum_{\ell=0}^{W-1} \boldsymbol{x}_{\text{data}}[t-\ell],
    \end{equation}
    with zero-padding outside the valid range.
    \item \textbf{Delay Weighted Sum}: each target token is a weighted sliding-window sum,
    \begin{equation}
        \boldsymbol{x}_{\text{target}}[t] = \sum_{\ell=0}^{W-1} \boldsymbol{G}[\ell]\boldsymbol{x}_{\text{data}}[t-\ell],
    \end{equation}
    where $\boldsymbol{G}[\ell]$ is a discrete Gaussian kernel
    \begin{equation}
        \boldsymbol{G}[\ell] = A \cdot \exp\left(-\frac{(\ell-\mu)^2}{2\sigma^2}\right).
    \end{equation}
\end{itemize}

All synthetic tasks are formulated as teacher-forced token prediction problems. The data tokens are sampled from $\{1,\cdots,m\}$, where $m$ is chosen for each task so that all sum or weighted-sum targets fit within the vocabulary after reserving token $0$ for padding and the final token for the go symbol. Delay Sum targets are integer-valued by construction, while Delay Weighted Sum targets are rounded to the nearest integer token after applying the Gaussian weights; configurations that would overflow the vocabulary are rejected. Validation accuracy is teacher-forced token accuracy on the supervised target segment.

We evaluate these mappings under two delay protocols. 
In the \emph{fixed-delay} protocol, all sequences in a task share the same padding lag. 
In the \emph{variable-delay} protocol, each sample draws its padding lag uniformly from a task-specific interval. 
The latter directly tests whether a model can remain effective when the relevant retrieval location varies across examples rather than being captured by a single shared LTI kernel.
Consequently, solving these tasks requires the model to look back to the correct long-range position and recover the needed information, making them a strong test of sequence memory.

All synthetic models are trained from scratch on each task and optimized with AdamW using a fixed learning-rate schedule without weight decay. 
For fairness, we keep the number of layers fixed and match total parameter counts across baselines by adjusting the hidden size $H$. 
We use the default hidden-size/intermediate-size relation for each architecture family whenever possible: 
for LSTM and S4-style DSSM blocks, the intermediate size equals the hidden size; 
for Mamba, the intermediate size is $2H$;
for the GPT-2 recipe, the feed-forward size is $4H$;
for the Llama recipe, we use a SwiGLU-style feed-forward size with a slight task-specific adjustment when needed to keep the parameter budget aligned while preserving an even per-head dimension for RoPE. 
Within the S4D family, we align the state size task-wise: Delay Copy uses $N=64$, while Delay Sum and Delay Weighted Sum use $N=128$.
For both S4D and DS4D, we initialize the SSM backbone with the S4D-Lin scheme~\citep{gu2022parameterization}; DS4D then applies the delayed-branch and delay-location initialization described in Sec.~\ref{sec:architecture}.
For the synthetic DS4D runs, we use the default delay initialization: delays are initialized on a uniform grid and randomly assigned to channels, the initial branch fraction is set to $p_n=0.1$, and the phase is initialized near $\pi$ with a small random perturbation.
For each baseline, we sweep the base learning rate over $\{2\times 10^{-4}, 10^{-3}\}$ and report the best result.
For DS4D, we keep the $\Delta t$ learning rate fixed at $1\times 10^{-4}$ across all synthetic runs, since we found the model noticeably more sensitive to $\Delta t$ updates than to the base learning rate.
Since the GPT-2 recipe uses learned absolute positional embeddings, its parameter count scales with the maximum context length; we therefore use task-specific hidden sizes for its fixed-delay and variable-delay settings to keep the total parameter budget aligned. 
The other baselines reuse the same hidden size across the two protocols.

\begin{table}[htbp]
    \centering
    \caption{Fixed-delay synthetic task setups and training budgets.}
    \label{tab:synthetic-fixed-hparams}
    \begin{tabular}{lccc}
        \toprule
        \textsc{Parameter} & \textsc{Delay Copy} & \textsc{Delay Sum} & \textsc{Delay Weighted Sum}\\
        \midrule
        Train Size & $128{,}000$ & $512{,}000$ & $512{,}000$ \\
        $\boldsymbol{x}_{\text{data}}$ length & $32$ & $32$ & $32$ \\
        Padding lag $L_{\text{lag}}$ & $1024$ & $1024$ & $1024$ \\
        Window $W$ & -- & $16$ & $16$ \\
        Gaussian Weight $A$ & -- & -- & $0.75$ \\
        Gaussian Center $\mu$ & -- & -- & $8.0$ \\
        Gaussian Width $\sigma$ & -- & -- & $12.0$ \\
        Batch Size & $64$ & $64$ & $64$ \\
        Max Steps & $2000$ & $8000$ & $8000$ \\
        \bottomrule
    \end{tabular}
\end{table}

\begin{table}[htbp]
    \centering
    \caption{Variable-delay synthetic task setups and training budgets. Padding lags are sampled uniformly for each sequence.}
    \label{tab:synthetic-variable-hparams}
    \begin{tabular}{lccc}
        \toprule
        \textsc{Parameter} & \textsc{Variable Copy} & \textsc{Variable Sum} & \textsc{Variable Weighted Sum}\\
        \midrule
        Train size & $512{,}000$ & $512{,}000$ & $512{,}000$ \\
        $\boldsymbol{x}_{\text{data}}$ length & $32$ & $32$ & $32$ \\
        Padding lag range & $[1014, 1024]$ & $[507, 512]$ & $[507, 512]$ \\
        Window $W$ & -- & $16$ & $16$ \\
        Gaussian weight $A$ & -- & -- & $0.75$ \\
        Gaussian center $\mu$ & -- & -- & $8.0$ \\
        Gaussian width $\sigma$ & -- & -- & $12.0$ \\
        Batch size & $64$ & $64$ & $64$ \\
        Max steps & $8000$ & $8000$ & $8000$ \\
        \bottomrule
    \end{tabular}
\end{table}

Table~\ref{tab:synthetic-backbones} summarizes the core backbone settings used across the three synthetic benchmarks. We only report the most salient architectural knobs here and omit task-independent optimizer details.

\begin{table}[htbp]
    \centering
    \caption{
        Core baseline settings and total parameter counts across the three synthetic benchmarks. 
        All Delay Copy backbones use $2$ layers and full context length $1088$; all Delay Sum and Delay Weighted Sum backbones use $3$ layers, with context/history length $1088$ in the fixed-delay setting and $576$ in the variable-delay setting. 
        Hidden sizes are shared across fixed and variable settings except the GPT-2 recipe, whose learned absolute positional embeddings make its parameter count depend on the context length. 
        Family-specific hidden-size/intermediate-size relations follow their default settings whenever possible: LSTM and S4-style DSSM blocks use intermediate size equal to the hidden size, Mamba uses intermediate size $2H$ with convolution kernel $3$, the GPT-2 recipe uses feed-forward size $4H$, and the Llama recipe uses a SwiGLU-style feed-forward size with a small task-specific adjustment to preserve an even per-head dimension for RoPE while keeping parameter counts aligned. 
        Within the S4D family, Delay Copy uses state size $64$, while Delay Sum and Delay Weighted Sum use state size $128$. 
        DS4D uses a smaller matched-parameter hidden size than S4D because of its additional delay parameters, while Mamba keeps its default architecture and tunes state size only to best match the total parameter budget.}
    \label{tab:synthetic-backbones}
    \begin{tabular}{l c c c c r}
        \toprule
        \textsc{Model} & \textsc{\#Layers} & \textsc{Hidden Size} & \textsc{\#Heads} & \textsc{State Size} & \textsc{\#Params} \\
        \midrule
        \multicolumn{6}{c}{\textsc{Delay Copy}} \\
        LSTM & 2 & 96 & -- & -- & 168.4k \\
        S4D & 2 & 143 & -- & 64 & 167.3k \\
        Mamba & 2 & 88 & -- & 48 & 167.8k \\
        GPT-2 recipe & 2 & 60 & 4 & -- & 165.4k \\
        Llama recipe & 2 & 80 & 8 & -- & 169.8k \\
        DS4D & 2 & 75 & -- & 64 & 167.4k \\
        \midrule
        \multicolumn{6}{c}{\textsc{Delay Sum}} \\
        LSTM & 3 & 177 & -- & -- & 862.7k \\
        S4D & 3 & 260 & -- & 128 & 866.3k \\
        Mamba & 3 & 168 & -- & 76 & 867.6k \\
        GPT-2 recipe (fixed) & 3 & 132 & 2 & -- & 855.5k \\
        GPT-2 recipe (variable) & 3 & 140 & 2 & -- & 876.0k \\
        Llama recipe & 3 & 150 & 3 & -- & 868.7k \\
        DS4D & 3 & 138 & -- & 128 & 870.5k \\
        \midrule
        \multicolumn{6}{c}{\textsc{Delay Weighted Sum}} \\
        LSTM & 3 & 177 & -- & -- & 862.7k \\
        S4D & 3 & 260 & -- & 128 & 866.3k \\
        Mamba & 3 & 168 & -- & 76 & 867.6k \\
        GPT-2 recipe (fixed) & 3 & 132 & 2 & -- & 855.5k \\
        GPT-2 recipe (variable) & 3 & 140 & 2 & -- & 876.0k \\
        Llama recipe & 3 & 150 & 3 & -- & 868.7k \\
        DS4D & 3 & 138 & -- & 128 & 870.5k \\
        \bottomrule
    \end{tabular}
\end{table}

\begin{table}[htbp]
    \centering
    \small
    \setlength{\tabcolsep}{4.5pt}
    \caption{Best validation accuracy on synthetic tasks. Values are reported in $[0,1]$ as mean $\pm$ std over three seeds, where each baseline selects its best base learning rate. Best results are bolded, and second-best results are underlined.}
    \label{tab:synthetic-acc}
    \resizebox{\textwidth}{!}{
    \begin{tabular}{lcccccc}
        \toprule
        & \multicolumn{2}{c}{\textsc{Delay Copy}} & \multicolumn{2}{c}{\textsc{Delay Sum}} & \multicolumn{2}{c}{\textsc{Delay Weighted Sum}} \\
        \cmidrule(lr){2-3} \cmidrule(lr){4-5} \cmidrule(lr){6-7}
        \textsc{Model} & \textsc{Fixed} & \textsc{Variable} & \textsc{Fixed} & \textsc{Variable} & \textsc{Fixed} & \textsc{Variable} \\
        \midrule
        LSTM & \underline{$0.01 \pm 0.00$} & $0.01 \pm 0.00$ & $0.06 \pm 0.00$ & $0.06 \pm 0.00$ & $0.07 \pm 0.00$ & $0.07 \pm 0.00$ \\
        S4D & \textbf{\boldmath$1.00 \pm 0.00$} & $0.85 \pm 0.03$ & $0.82 \pm 0.06$ & $0.59 \pm 0.03$ & $0.71 \pm 0.03$ & $0.60 \pm 0.02$ \\
        Mamba & \underline{$0.01 \pm 0.00$} & $0.06 \pm 0.08$ & $0.39 \pm 0.04$ & $0.48 \pm 0.00$ & $0.38 \pm 0.01$ & $0.41 \pm 0.02$ \\
        DS4D & \textbf{\boldmath$1.00 \pm 0.00$} & \underline{$0.95 \pm 0.00$} & \textbf{\boldmath$0.99 \pm 0.00$} & $0.76 \pm 0.04$ & \textbf{\boldmath$0.80 \pm 0.01$} & \underline{$0.68 \pm 0.03$} \\
        \midrule
        GPT-2 recipe & \textbf{\boldmath$1.00 \pm 0.00$} & \textbf{\boldmath$1.00 \pm 0.00$} & \underline{$0.96 \pm 0.04$} & \underline{$0.86 \pm 0.11$} & $0.74 \pm 0.03$ & $0.67 \pm 0.04$ \\
        Llama recipe & \textbf{\boldmath$1.00 \pm 0.00$} & \textbf{\boldmath$1.00 \pm 0.00$} & \textbf{\boldmath$0.99 \pm 0.00$} & \textbf{\boldmath$0.89 \pm 0.15$} & \underline{$0.75 \pm 0.01$} & \textbf{\boldmath$0.72 \pm 0.03$} \\
        \bottomrule
    \end{tabular}
    }
\end{table}

\subsection{Efficiency and FFT Accuracy}
\label{sec:efficiency-details}

All efficiency experiments use a $4$-layer backbone with state size $64$, explicit $\boldsymbol{B}$, pre-norm, GLU final activation, float32 precision, and CUDA execution on a single NVIDIA GeForce RTX 5090 GPU. 
We compare the S4D baseline against two DS4D variants: a hidden-matched model with hidden size $256$, and a matched-parameter model with hidden size $241$. 
S4D uses the official CUDA kernel for full-sequence training; streaming decode follows the official recurrent step path, which has no dedicated CUDA decode kernel. 
DS4D uses our custom CUDA fast paths. 
All benchmarks use sequence lengths $\{1024, 2048, 4096, 8192\}$.

Training throughput uses batch size $32$, one burn-in run, $10$ warmup steps, and $50$ timed steps, reported in tokens per second; training memory reuses the same setup but records peak allocated CUDA memory over $20$ timed steps. 
Autoregressive decoding throughput uses batch size $1$, one burn-in run, $32$ warmup steps, and $256$ timed one-token updates after cache prefilling, with retained history fixed at $8192$ so that the trend reflects per-token compute rather than cache growth. 
For decoding memory, the retained history size is instead set equal to the current sequence length, and we report peak allocated memory together with resident cache size.

FFT accuracy compares the same discretized DS4D model evaluated in FFT mode and recurrent mode on identical inputs over $5$ random seeds and the same sequence lengths. 
For these FFT runs, we use $M=\mathrm{nextpow2}(2L)$ and set the contour-shift target to $\varepsilon_{\mathrm{shift}}=10^{-3}$, equivalently $\rho^M=10^{-3}$.
For each seed, we instantiate paired models from the same config, copy the state dict from the FFT-mode model to the recurrent-mode model, and evaluate both on the same randomly generated input sequence. 
Table~\ref{tab:fft-scan-accuracy} reports the per-length summary on the output hidden states: the relative RMS error, normalized by the RMS magnitude of the recurrent-mode output, stays on the order of $1.0 \times 10^{-5}$, and the maximum absolute error stays below $1.2 \times 10^{-4}$.

\begin{table}[htbp]
    \centering
    \footnotesize
\caption{FFT-versus-recurrent accuracy for the same discretized DS4D model. Metrics are computed on the output hidden states; relative RMSE is normalized by the RMS of the recurrent output. Entries use the shorthand $(a \pm b)e\text{-}c = (a \pm b)\times 10^{-c}$ and report mean $\pm$ std over $5$ seeds.}
    \label{tab:fft-scan-accuracy}
    \resizebox{\linewidth}{!}{
    \begin{tabular}{lcccc}
        \hline
        Metric& 1024 & 2048 & 4096 & 8192 \\
        \hline
        Relative RMSE & $(1.37 \pm 0.04)e\text{-}5$ & $(1.12 \pm 0.01)e\text{-}5$ & $(1.07 \pm 0.00)e\text{-}5$ & $(1.48 \pm 0.01)e\text{-}5$ \\
        Max abs.\ error & $(9.43 \pm 0.53)e\text{-}5$ & $(7.49 \pm 0.56)e\text{-}5$ & $(7.19 \pm 0.37)e\text{-}5$ & $(1.12 \pm 0.03)e\text{-}4$ \\
        \hline
    \end{tabular}
    }
\end{table}

\subsection{General Sequence Modeling}

Our general-sequence experiments cover simple multimodal sequence tasks together with Long Range Arena (LRA). Following the evaluation style of S4D~\citep{gu2022parameterization}, we use these tasks to probe whether DS4D retains broad sequence-modeling ability beyond the synthetic delayed-retrieval setting.

\paragraph{Baseline Results.}
Except for sMNIST, all baseline results in the general-sequence benchmarks are taken directly from the S4D paper~\citep{gu2022parameterization}. 
All DS4D results in this section are computed over $2$ or $3$ random seeds.

\paragraph{sMNIST.}
Sequential MNIST (sMNIST) treats each image as a length-$784$ sequence.
We use a unidirectional Delay S4D classifier with mean pooling and fix the delay initialization to one image row, i.e., $\tau_0 = 28$. 
Because sMNIST does not present a strong delayed-retrieval problem, we initialize the delayed branch with a small gain $p_n = 0.015$ and phase center $\omega = 0$.

\paragraph{SC.}
Speech Commands (SC) is a raw-audio keyword-spotting benchmark in which each input is a $16000$-sample waveform and the task is $35$-way spoken-word classification.
Following the S4D setup, we report both \textsc{SC (AR)} and \textsc{SC (Bi.)}: the former uses a unidirectional model with last pooling to mimic causal/autoregressive readout, while the latter uses a bidirectional model with mean pooling over the full sequence.

\paragraph{BIDMC.}
BIDMC is a medical time-series regression benchmark based on paired electrocardiogram (EKG) and photoplethysmogram (PPG) signals of length $4000$.
The three official targets are respiratory rate (RR), heart rate (HR), and blood oxygen saturation (SpO2).

\paragraph{LRA.}
Long Range Arena (LRA)~\citep{tay2020long} evaluates efficient sequence models on a diverse suite of long-context classification tasks. Following the task summaries used in the S4D/S5 literature, we consider:
\begin{itemize}
    \item \textbf{ListOps.} A length-$2000$ hierarchical arithmetic-expression task that predicts the value of a nested operator sequence and primarily probes compositional parsing over long token streams.
    \item \textbf{Text.} Character-level IMDb sentiment classification on sequences up to roughly $4000$ characters, testing whether the model can compose sentiment evidence over long textual contexts.
    \item \textbf{Retrieval.} The AAN document-matching task, where the model receives two long documents and predicts whether they belong together; this benchmark mainly stresses long-range information compression across paired sequences.
    \item \textbf{Image.} Grayscale CIFAR-10 classification in image-as-sequence form, where the model processes the image as a long pixel stream rather than through a 2D convolutional backbone.
    \item \textbf{Pathfinder.} A binary visual reasoning task that asks whether two marked endpoints are connected by a thin path in a serialized image.
    \item \textbf{Path-X.} The longer and harder Pathfinder variant with substantially larger sequence length, designed to probe extremely long-range visual reasoning.
\end{itemize}

\paragraph{Training Setup.}
For all tasks, we compare DS4D against the S4D baseline using the same backbone depth and closely matched model size, while preserving each benchmark shell whenever possible. The shared training choices are:
\begin{itemize}
    \item \textbf{Scheduler.} We use cosine scheduler for all tasks except BIDMC, which follows the S4D paper in using a multi-step schedule.
    \item \textbf{Weight Decay.} Following the S4D paper, all tasks use weight decay $0.05$ except Pathfinder, which uses $0.03$.
    \item \textbf{Input Projection.} All runs use a trainable explicit $\boldsymbol{B}$.
    \item \textbf{Parameter Matching.} We retune the hidden size $H$ so that the number of trainable parameters stays within about $1\%$ of the corresponding S4D configuration.
    \item \textbf{SSM Backbone Initialization.} All general-sequence runs use the S4D-Inv initialization for the shared SSM backbone, since the S4D paper reports that S4D-Lin does not successfully solve Path-X.
    \item \textbf{Delay Initialization.} We use constant delay only for sMNIST and LRA Image, with $\tau \equiv 28$ and $\tau \equiv 32$ respectively; all other tasks use the default uniform-grid initialization with random channel assignment.
    \item \textbf{Learning Rates.} We keep the main learning rate aligned with the S4D recipe and use a smaller learning rate only for $\Delta t$. Unlike S4D, DS4D keeps $\boldsymbol{A}$, $\boldsymbol{A}_d$, and $\boldsymbol{B}$ on the main learning-rate group rather than assigning them a reduced learning rate.
    \item \textbf{Directionality.} Except for sMNIST and SC (AR), all general-sequence runs use bidirectional models.
\end{itemize}
Table~\ref{tab:general-delay-s4d-configs} summarizes the benchmark-specific Delay S4D configurations used for these runs. The BIDMC endpoints share the same depth, width, state size, and schedule, but differ slightly in normalization and delayed-branch initialization as shown in the table. The \textsc{Text}, \textsc{Retrieval}, and \textsc{Image} rows correspond to the char-level IMDB, AAN, and grayscale CIFAR-10 wrappers in the codebase, respectively.

\begin{table*}[htbp]
    \centering
    \footnotesize
    \setlength{\tabcolsep}{4.2pt}
    \caption{Delay S4D hyperparameters for the general sequence modeling benchmarks. 
    All runs use the S4D-Inv backbone initialization. 
    Here $H$ denotes the hidden size and $N$ denotes the per-channel state size.}
    \label{tab:general-delay-s4d-configs}
    \resizebox{\textwidth}{!}{
    \begin{tabular}{lcccccccccccc}
        \toprule
        \textsc{Task} & \textsc{Depth} & \textsc{$H$} & \textsc{$N$} & \textsc{$p_n$} & \textsc{$\omega$} & \textsc{Norm} & \textsc{Pre-Norm} & \textsc{Batch} & \textsc{LR} & \textsc{$\Delta t$ LR} & \textsc{WD} & \textsc{Epochs} \\
        \midrule
        \multicolumn{13}{c}{\textsc{Multimodal Tasks}} \\
        \midrule
        \textsc{sMNIST} & 6 & 115 & 64 & 0.015 & 0 & BN & True & 64 & $1.0\mathrm{e}{-3}$ & $1.0\mathrm{e}{-4}$ & $0.05$ & 20 \\
        \textsc{SC (AR)} & 6 & 114 & 64 & 0.1 & $\pi$ & BN & True & 16 & $1.0\mathrm{e}{-2}$ & $1.0\mathrm{e}{-4}$ & $0.05$ & 40 \\
        \textsc{SC (Bi.)} & 6 & 116 & 64 & 0.1 & $\pi$ & BN & True & 16 & $1.0\mathrm{e}{-2}$ & $1.0\mathrm{e}{-4}$ & $0.05$ & 40 \\
        \textsc{BIDMC-RR} & 6 & 101 & 256 & 0.1 & $\pi$ & LN & True & 32 & $1.0\mathrm{e}{-2}$ & $1.0\mathrm{e}{-4}$ & $0.05$ & 500 \\
        \textsc{BIDMC-HR} & 6 & 101 & 256 & 0.05 & $\pi$ & BN & True & 32 & $1.0\mathrm{e}{-2}$ & $1.0\mathrm{e}{-4}$ & $0.05$ & 500 \\
        \textsc{BIDMC-SpO2} & 6 & 101 & 256 & 0.1 & $\pi$ & BN & True & 32 & $1.0\mathrm{e}{-2}$ & $1.0\mathrm{e}{-4}$ & $0.05$ & 500 \\
        \midrule
        \multicolumn{13}{c}{\textsc{LRA Tasks}} \\
        \midrule
        \textsc{ListOps} & 8 & 116 & 64 & 0.1 & $\pi$ & BN & False & 50 & $1.0\mathrm{e}{-2}$ & $5.0\mathrm{e}{-4}$ & $0.05$ & 40 \\
        \textsc{Text} & 6 & 242 & 64 & 0.1 & $\pi$ & BN & True & 16 & $1.0\mathrm{e}{-2}$ & $1.0\mathrm{e}{-4}$ & $0.05$ & 32 \\
        \textsc{Retrieval} & 6 & 242 & 64 & 0.1 & $\pi$ & BN & True & 64 & $1.0\mathrm{e}{-2}$ & $1.0\mathrm{e}{-4}$ & $0.05$ & 20 \\
        \textsc{Image} & 6 & 497 & 64 & 0.01 & 0 & LN & False & 50 & $1.0\mathrm{e}{-2}$ & $1.0\mathrm{e}{-4}$ & $0.05$ & 200 \\
        \textsc{Pathfinder} & 6 & 243 & 64 & 0.1 & $\pi$ & BN & True & 64 & $4.0\mathrm{e}{-3}$ & $5.0\mathrm{e}{-4}$ & $0.03$ & 200 \\
        \textsc{Path-X} & 6 & 243 & 64 & 0.1 & $\pi$ & BN & True & 16 & $5.0\mathrm{e}{-4}$ & $1.0\mathrm{e}{-4}$ & $0.05$ & 50 \\
        \bottomrule
\end{tabular}
    }
\end{table*}

\paragraph{Approximate Training Time.}
Table~\ref{tab:general-delay-s4d-time} reports rough single-seed wall-clock estimates for the DS4D benchmarks measured on a single NVIDIA GeForce RTX 5090 (32GB) GPU worker.

\begin{table}[htbp]
    \centering
    \footnotesize
    \setlength{\tabcolsep}{5pt}
    \caption{Approximate training time for DS4D on the general sequence modeling benchmarks.}
    \label{tab:general-delay-s4d-time}
    \begin{tabular}{lc}
        \toprule
        \textsc{Task} & \textsc{Approx. time} \\
        \midrule
        \textsc{sMNIST} & $\sim 0.3$ h \\
        \textsc{SC (AR)} & $\sim 5.8$ h \\
        \textsc{SC (Bi.)} & $\sim 11.1$ h \\
        \textsc{BIDMC} & $\sim 2.8$ h \\
        \midrule
        \textsc{ListOps} & $\sim 2.6$ h \\
        \textsc{Text} & $\sim 1.6$ h \\
        \textsc{Retrieval} & $\sim 7.7$ h \\
        \textsc{Image} & $\sim 7.6$ h \\
        \textsc{Pathfinder} & $\sim 14.8$ h \\
        \textsc{Path-X} & $\sim 55$ h \\
        \bottomrule
    \end{tabular}
\end{table}

\subsection{External Assets and Official Links}
\label{sec:external-assets}

We rely on standard public benchmarks and official upstream codebases. 
Table~\ref{tab:external-assets} lists the primary external assets used in this paper together with their official source pages and the license or usage terms currently stated there.

\begin{table*}[htbp]
    \centering
    \footnotesize
    \setlength{\tabcolsep}{5pt}
    \caption{Primary external assets used in this work.}
    \label{tab:external-assets}
    \resizebox{\textwidth}{!}{
    \begin{tabular}{p{3.0cm}p{4.0cm}p{4.8cm}p{3.8cm}}
        \toprule
        \textsc{Asset} & \textsc{Role in this paper} & \textsc{Official source} & \textsc{License / terms stated by source} \\
        \midrule
        \textsc{state-spaces/s4} & Official upstream S4/S4D codebase used for baseline comparison and reproduction & \texttt{state-spaces/s4} (GitHub) & Apache License 2.0 \\
        \textsc{MNIST} & Data source for sMNIST & Keras MNIST dataset page & CC BY-SA 3.0 \\
        \textsc{Speech Commands} & Data source for Speech Commands benchmark & Google Speech Commands dataset card & CC BY 4.0 \\
        \textsc{BIDMC} & Data source for BIDMC RR / HR / SpO2 benchmarks & PhysioNet BIDMC dataset page & Open Data Commons Attribution License v1.0 \\
        \textsc{Long Range Arena} & Benchmark suite used for ListOps, Text, Retrieval, Image, Pathfinder, and Path-X & \texttt{google-research/long-range-arena} (GitHub) & Apache License 2.0 \\
        \bottomrule
    \end{tabular}
    }
\end{table*}

\section{Additional Results}
\label{sec:additional-results}

Figure~\ref{fig:delay-s4d-heatmaps} provides a complementary view of the six synthetic DS4D checkpoints analyzed in Table~\ref{tab:synthetic-mechanism}.
After sorting channels by learned delay, all six tasks exhibit broad diagonal bands spanning the available history horizon, consistent with DS4D preserving a distributed delay bank rather than collapsing to a single dominant lag.

\begin{figure*}[htbp]
    \centering
    \includegraphics[width=\textwidth]{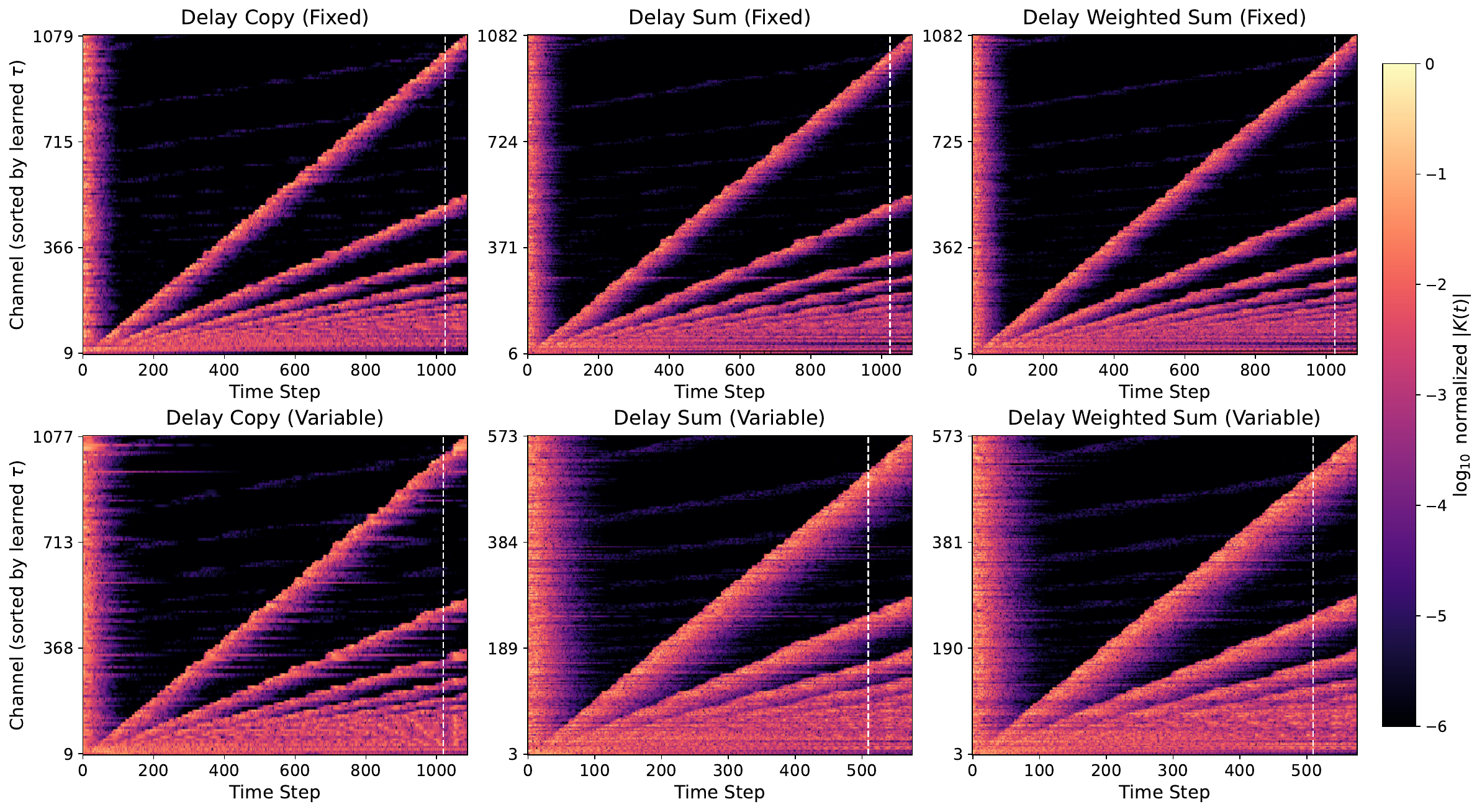}
    \caption{First-layer impulse-response heatmaps for the six synthetic DS4D checkpoints.
    Rows are channels sorted by learned delay; y-tick labels report the corresponding learned delay values after sorting.
    The white dashed line marks the padding lag $L_{\text{lag}}$.
    Colors show independently normalized $\log_{10}$ impulse-response magnitude.
    The diagonal bands indicate a distributed family of delay-sensitive channels spanning the available history range.}
    \label{fig:delay-s4d-heatmaps}
\end{figure*}

For comparison, Figure~\ref{fig:official-s4d-heatmaps} shows the corresponding first-layer impulse-response heatmaps for S4D on the same six tasks, using the best-validation checkpoint available for each task.
Sorted by empirical peak-response time rather than by a learned delay variable, these kernels are organized mainly by intrinsic response timescale and do not exhibit the same task-aligned diagonal bands.

\begin{figure*}[htbp]
    \centering
    \includegraphics[width=\textwidth]{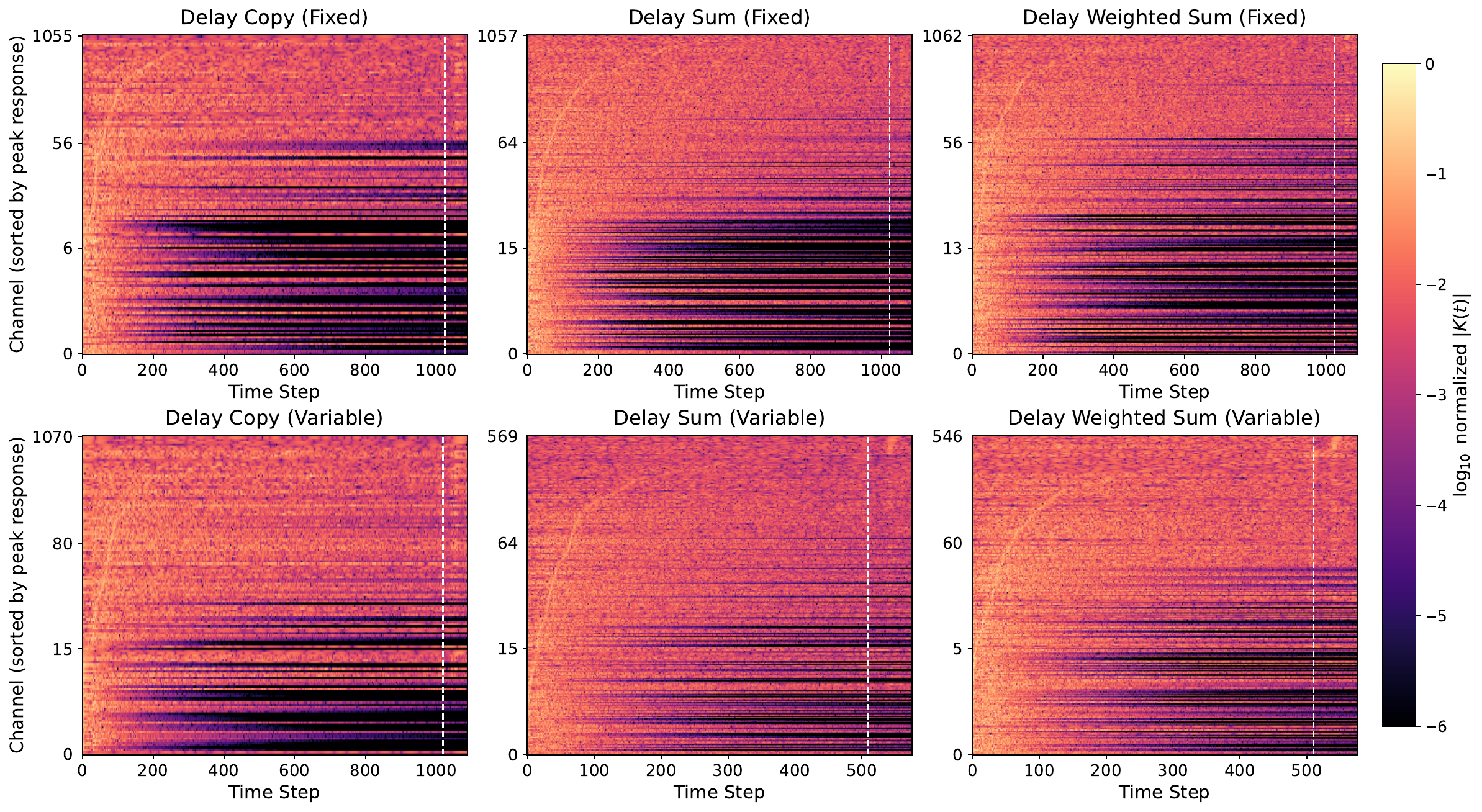}
    \caption{First-layer impulse-response heatmaps for the corresponding S4D checkpoints on the six synthetic tasks, using the best-validation checkpoint available for each task.
    Rows are channels sorted by peak-response time; y-tick labels report the corresponding peak-response times after sorting.
    The white dashed line marks the padding lag $L_{\text{lag}}$ or the center of the variable-delay interval.
    Colors show independently normalized $\log_{10}$ impulse-response magnitude.
    The predominantly horizontal organization reflects a spread of implicit timescales rather than explicit delay-aligned reinjection.}
    \label{fig:official-s4d-heatmaps}
\end{figure*}


\end{document}